\documentclass[a4paper,11pt]{amsart}

\usepackage{mathtools}
\numberwithin{equation}{section}

\usepackage{pdfsync}
\usepackage{dsfont}
\usepackage{xcolor}
\usepackage{cancel}
\usepackage{esint}
\usepackage{amsthm}
\usepackage{enumerate}
\usepackage{amsfonts}
\usepackage{amssymb}%
\usepackage{mathtools}
\usepackage{braket}
\usepackage{bm}
\usepackage{amsmath}
\usepackage{graphicx}
\usepackage{caption}
\usepackage{subcaption}
\usepackage{fullpage}

\usepackage{hyperref}

\newcommand{\measurerestr}{%
	\,\raisebox{-.127ex}{\reflectbox{\rotatebox[origin=br]{-90}{$\lnot$}}}\,%
}

\usepackage{amsfonts}
\usepackage{graphicx}%

\providecommand{\U}[1]{\protect\rule{.1in}{.1in}}

\newtheorem{theorem}{Theorem}[section]

\newtheorem{corollary}[theorem]{Corollary}

\newtheorem{example}[theorem]{Example}

\newtheorem{lemma}[theorem]{Lemma}

\newtheorem{proposition}[theorem]{Proposition}
\newtheorem{remark}[theorem]{Remark}

\newtheorem{assumption}[theorem]{Assumption}

\providecommand{\keywordsnew}[1]{\textbf{\textit{\noindent  Keywords and phrases. }} #1}

\newcommand{\veps}{\varepsilon}
\newcommand{\e}{\varepsilon}

\newcommand{\R}{\mathbb{R}}
\newcommand{\HH}{\mathcal{H}}

\DeclareMathOperator{\Prob}{\mathbb{P}}
\newcommand{\M}{\mathcal{M}}

\newcommand{\red}{\color{black}}
\newcommand{\blue}{\color{black}}

\definecolor{mygreen}{rgb}{0.1,0.75,0.2}

\newcommand{\nc}{\normalcolor}
\newcommand{\E}{\mathbb{E}}

\renewcommand{\P}{\mathbb{P}}
\newcommand{\Per}{\text{Per}}

\newcommand{\bnu}{\mathbf{\nu}}
\newcommand{\nrho}[1]{w_{#1}\rho_{#1}}
\newcommand{\trho}{\rho}

\title[Geometric flows for adversarial classification]{Adversarial Classification: Necessary Conditions and Geometric Flows}
\author{Nicol\'as Garc\'ia Trillos}
\address{Department of Statistics, University of Wisconsin, Madison, Wisconsin, USA}
\email{garciatrillo@wisc.edu}
\author{Ryan Murray}
\address{Department of Mathematics, North Carolina State University, Raleigh, NC, USA}
\email{rwmurray@ncsu.edu}

\thanks{ NGT was supported by NSF grant DMS 2005797}

\begin{document}

\maketitle

%

\begin{abstract}%
We study a version of adversarial classification where an adversary is empowered to corrupt data inputs up to some distance $\e$, using tools from variational analysis. In particular, we describe necessary conditions associated with the optimal classifier subject to such an adversary. Using the necessary conditions, we derive a geometric evolution equation which can be used to track the change in classification boundaries as $\e$ varies. This evolution equation may be described as an uncoupled system of differential equations in one dimension, or as a mean curvature type equation in higher dimension. In one dimension, and under mild assumptions on the data distribution, we rigorously prove that one can use the initial value problem starting from $\e=0$, which is simply the Bayes classifier, in order to solve for the global minimizer of the adversarial problem for small values of $\veps$. \blue In higher dimensions we provide a similar result, albeit conditional to the existence of regular solutions of the initial value problem. In the process of proving our main results we obtain a result of independent interest connecting the original adversarial problem with an optimal transport problem under no assumptions on whether classes are balanced or not. \nc Numerical examples illustrating these ideas are also presented.
\end{abstract}
\keywordsnew{adversarial learning, classification, optimal transportation, geometric flow, differential equations, perimeter regularization}

\section{Introduction}
 In many learning settings, and in particular in the setting of deep learning, classifiers are known to behave poorly when exposed to adversarial examples. This has led to a significant body of work studying both the construction of specific adversaries and possible algorithms defending against them. Furthermore, the notion of pitting learners versus adversaries has stimulated significant new algorithms such as generative adversarial networks. One may view such adversarial frameworks as one possible notion of robustness of a learning algorithm, a critical concern in many applications.

In this work we consider the problem of optimal adversarial learning and aim at connecting it with a family of geometric evolution equations. The evolution equations that we derive answer the question: how would the decision boundary of a robust classifier change infinitesimally, if the adversary was to infinitesimally increase its power to perturb the data? Besides establishing new theoretical understanding for adversarial classification linking it with a set of geometric equations of surface diffusion type (similar to the ones describing the dynamics of interfaces of droplets of viscous fluids), our aim is also to explore computational alternatives to solve adversarial classification problems. At the theoretical level, a standard un-robust classification problem  admits an explicit solution (i.e. the Bayes classifier), while adversarial problems typically do not have explicit solutions and in general are quite challenging from a numerical point of view.

While the general perspective that we have described above can be studied in a variety of settings, here we will study a concrete model for adversarial binary classification. In particular, we assume that a binary classifier is subject to a data perturbing adversary: namely, that for any future input $x \in \R^d$ and associated output $y \in \{0,1\}$, the adversary may select a new associated input $\tilde x = x + \eta$ in order to disrupt a classifier.  The adversary is assumed to possess limited power, namely that $\|\eta\|_2 < \e$, but is assumed to have knowledge of the classifier that has been chosen.  A basic question is how such an adversary affects optimal classifiers. Various works have posited that adversaries do have an effect on classifiers, and that they can induce regular decision boundaries. Heuristically, from a geometric perspective this is natural, as boundaries with more surface area offer more opportunity for adversaries to disrupt classifiers. However, rigorous justification of this assertion is, to this point, unavailable. Several recent works have derived sufficient conditions for the adversarial learning problem with such an adversary. In particular, \cite{bhagoji2019lower,pydi2019adversarial} both derive a duality principle related to the optimal adversarial classifier. They use this to derive bounds on the effect on the loss of such an adversary. Such a duality principle provides an embedding of the optimal adversarial classification problem as an optimal transportation problem.


As mentioned earlier, despite the potential difficulty of solving the optimal adversarial classification problem for a fixed $\e>0$ via optimization, we notice that the solution of the problem for $\e=0$ is well-known and does not require optimization: the optimizer is the classical Bayes classifier. Namely, if we define
\[
  \nrho{0}(x) = \P(X \in dx, Y = 0),\qquad \nrho{1}(x) = \P(X \in dx, Y=1),
\]
then the Bayes classifier given by
\[
  u_0(x) = \begin{cases} 1 &\text{ if } \nrho{1}(x) > \nrho{0}(x) \\ 0 & \text{ otherwise} \end{cases}
\]
is known to be a minimizer of the un-robust risk. In the one dimensional case we expect to be able to write $u_0(x) = \mathds{1}_E$ for a set of the form $E = \cup_{i=1}^K [a_i(0),b_i(0)]$, where the ``$0$''  indicates that $\e = 0$.
The central idea of this work is to derive evolution equations for the decision boundary of an optimal classifier as $\e$ increases from zero, in the regime where we may construct optimal classifiers as a perturbation of the explicit Bayes classifier. This is achieved by deriving local necessary conditions (i.e. Euler-Lagrange type equations) for optimal adversarial classifiers for any fixed $\e$ \eqref{eqn:1D-Necessary}. In particular, in the one dimensional case, these necessary conditions take the form of the \emph{algebraic equation}
\[
  \nrho{1}(b_i(\e) - \e) = \nrho{0}(b_i(\e) + \e).
\]
Analogous necessary conditions are derived for the $a_i$. These necessary conditions are then used to derive  evolution equations \eqref{eqn:evolutiod-1d},\eqref{eqn:evolutiod-1db}. In particular, in one dimension this necessary condition takes the form of a decoupled, \emph{ordinary differential equation} (ODE)
\[
  \frac{db_i}{d\e} = -\frac{\nrho{0}'(b_i(\e) + \e) + \nrho{1}'(b_i(\e)-\e)}{\nrho{0}'(b_i(\e)+\e) - \nrho{1}'(b_i(\e) - \e)},
\]
with an analogous equation for the $a_i$. We remark that the resulting equation involves a sort of weak non-local algebraic condition, which in turn means the evolution equation includes a weak non-local forcing term. The evolution equation is ultimately a relatively simple decoupled ODE, which may then be solved directly using numerical solvers, with very modest computational effort and \emph{no optimization}. This gives an easily computed candidate solution to the optimal adversarial classification problem for $\e$ sufficiently close to zero.

As the equations that we derive are based upon necessary conditions, a natural question is whether solutions to the ODE indeed correspond to global minimizers of the optimal adversarial classification problems. Following  the duality principle derived in \cite{bhagoji2019lower}\cite{pydi2019adversarial} (which we extend here to include unbalanced classes, \blue and which holds under arbitrary metrics constraining adversarial perturbations and in arbitrary dimension\nc), we derive the following theorem (stated informally):

\begin{theorem}
  In one dimension, under mild technical assumptions on $\nrho{0},\nrho{1}$ and the associated Bayes classifier, there exists an interval $[0,\veps_0]$ such that the solution of the optimal adversarial classification problem is given by the solution to the decoupled differential equations \eqref{eqn:evolutiod-1d},\eqref{eqn:evolutiod-1db}  with initial values given by the decision boundary of the Bayes classifier (when $\e = 0$).
  \label{thm:informal}
\end{theorem}

Subsequently, we turn our attention to studying the problem in higher dimensions, where decision boundaries are now expressed as hyper-surfaces. For simplicity, we focus our attention on the setting where the adversarial constraints are given in terms of the standard Euclidean norm, which we denote by $|\cdot |$.  After deriving necessary conditions, which again take the form of weakly non-local algebraic equations \eqref{eqn:nD-necessary}, we derive an evolution equation for the decision boundary as $\e$ varies \eqref{eqn:d-dim-evolution}. \blue The well-posedness of this geometric evolution equation is not immediately clear, but under the assumption that regular solutions do exist we can also prove that the solution of the evolution equation characterizes global minimizers on some interval $[0,\e_0]$, see Theorem \ref{thm:Multid}. Using a Taylor expansion around $\veps=0$, we can also identify approximate geometric evolution dynamics which are more interpretable. In particular, we derive the evolution equation \nc
\eqref{eqn:d-dim-evolution-approx}, which may be written as follows:
\begin{equation}
  v(x) = -\frac{ \nabla \rho \cdot \nu + \rho \sum_i \kappa_i}{  (\nabla \nrho{1} - \nabla \nrho{0})\cdot \nu},
\label{eqn:WeightedMeanCurvature}
\end{equation}
where here $v$ represents the normal velocity (with respect to $\e$) of a point on the decision boundary of the Bayes classifier, $\nu$ is the normal vector to the boundary, $\kappa_i$ denote the principal curvatures \blue (see the Appendix for a definition) \nc of the boundary, and $\rho = \nrho{0} + \nrho{1} = \P(X \in dx)$. \blue Conceptually, the vector field $v \nu$ describes the infinitesimal change of the Bayes classifier (i.e. the minimizer of the problem when $\veps=0$) as the adversary increases its power. Evolution equation \eqref{eqn:WeightedMeanCurvature} takes the form of a weighted \textit{mean curvature flow} plus a biasing term (the biasing term is driven by the gradient of the distribution $\rho$). Mean curvature flow is an important geometric flow with many convenient properties, including a comparison principle, and is known in many instances to induce significant regularity to surfaces. In particular, mean curvature flow may be seen, within an appropriate function space, as a gradient flow of the perimeter functional (in particular a flow that aims at minimizing surface area). Equation \eqref{eqn:WeightedMeanCurvature} thus suggests that as $\e$ increases, the corresponding optimal decision boundaries become shorter and smoother, supporting previous work on the topic. In addition, at least for the unweighted case, there are powerful and efficient numerical algorithms to compute mean curvature flows (i.e. the MBO scheme \cite{merriman1992diffusion}).

\blue To be more concrete about the connection between equation \eqref{eqn:WeightedMeanCurvature} and perimeter minimization problems, let us consider the family of variational problems:
\begin{equation} 
\min_{E \subseteq \R^d} \{  R(\mathds{1}_E) + \veps \Per_\rho(E) \}
\label{eqn:PerimeterMinimization}
\end{equation}
indexed by $\veps\geq 0$, where $R$ denotes the standard average misclassification error and $\Per_{\rho}$ represents the weighted (by $\rho$) perimeter of the set $E$, which, for sets $E$ with smooth boundary $\partial E$, can be written as:
\[ \Per_\rho(E) := \int_{\partial E} \rho(x) d\mathcal{H}^{d-1}(x);   \]
in the above, $\mathcal{H}^{d-1}$ is the $d-1$ dimensional Hausdorff measure. Problem \eqref{eqn:PerimeterMinimization} can be interpreted as a regularized risk minimization problem over binary classifiers, where $\Per_\rho$ plays the role of an \textit{explicit regularizer}, in this case penalizing binary classifiers when they have large decision boundaries. Problem \eqref{eqn:PerimeterMinimization} is relevant in the context of adversarial learning because, as we illustrate formally in Section \ref{sec:Perimeter}, Equation \eqref{eqn:WeightedMeanCurvature} also describes the infinitesimal change of solutions to the family of problems \eqref{eqn:PerimeterMinimization} (indexed by $\veps$) when starting at $\veps=0$ (i.e. when starting with the Bayes classifier, which is the minimizer of the risk $R$.) From this observation we can deduce that the instantaneous regularization effect that the adversary has on the Bayes classifier is the same as the infinitesimal regularization effect enforced by explicit perimeter regularization. This observation provides a novel geometric interpretation for the role of adversaries in binary classification: they are approximately equivalent to an explicit perimeter penalization. This line of research has been further explored by the authors in their work with Bungert: \cite{BungertNGT-RM}, where they prove an equivalence between adversarial learning for binary classification and regularized risk minimization for all $\veps>0$ (and not just infinitesimally around $\veps=0$) at the expense of having to modify the notion of perimeter used to measure the size of the boundary of a set. \nc

\blue
In summary, in this paper we take a novel approach and view an adversarial problem as an ensemble of problems indexed by a parameter controlling the ability of an adversary to perturb the data. The main motivation for doing this is to provide new theoretical insights into the role played by adversaries in the training of binary classifiers. In concrete terms, we discuss properties of the evolution equations that track solutions to an ensemble of adversarial problems, starting from an un-robust optimal classifier. These evolution equations take the form of geometric equations. For the specific adversarial model that we study here the adversarial problem and its corresponding geometric evolution equations can be connected to a dual optimal transport problem, which is of interest on its own right and that extends earlier work in \cite{pydi2019adversarial} where the case of balanced labels $(w_0= w_1)$ was considered. In this paper, the connection to optimal transport is used to certify global optimality of the decision boundaries generated by the geometric flows. 
\nc

The remainder of this work is organized as follows. In Section \ref{sec:lit} we review some relevant literature. In Section \ref{sec:setup} we describe concretely the model that we consider. In Section \ref{sec:duality} we review and extend the duality principle related to the model. In Sections \ref{sec-1D} and \ref{sec:Global} we derive the main results in one dimension. Subsequently, Section \ref{sec:high-D} formally studies the higher-dimensional case. Finally, Section \ref{sec:conclusion} concludes by summarizing our work and describing a number of promising future directions.

\section{Related literature}\label{sec:lit}

\subsection{Adversarial learning}
A significant body of recent work considers the problem of adversarial learning; we only aim to provide a review of the most relevant references. Early works focused on the existence of adversarial examples in deep learning \cite{szegedy2013intriguing,goodfellow2014explaining}. These examples typically involved adding carefully structured noise to images in ways that was imperceptible to humans, but which led to gross classification errors for fitted neural networks. A number of different algorithms were then developed for both constructing adversarial attacks and defending against them; these models are distinct from but related to the one we consider in this work: \blue one influential example from this literature is \cite{madry2017towards}, which established important benchmarks for both adversarial attack and defense. \nc Several works advocate for attempting to differentiate between ``natural'' and ``adversarial'' inputs \cite{gong2017adversarial,grosse2017statistical,metzen2017detecting}, while other works describe the ability of adversaries to circumvent such a defense \cite{carlini2017adversarial,athalye2018obfuscated}. A parallel line of work posed a construction of improved classifiers by posing a game in which adversaries and classifiers iteratively try to best one another: this is the underlying framework for generative adversarial networks \cite{goodfellow2014generative}. 

One work along this vein which relates closely with our work is \cite{moosavi2019robustness}. That work observes that many boundaries obtained via robust classification are empirically observed to have smaller curvature. They then propose including a regularization term in classification that penalizes boundaries with higher curvature. Our work complements theirs in that we directly obtain a mean curvature in our d-dimensional evolution equation, indicating that the curvature indeed plays an explicit role in how decision boundaries change upon introducing stronger adversaries. While we do not explicitly prove that lower curvature is induced in our adversarial setting, the evolution equation implicitly suggests that such is the case, and a rigorous connection between these notions is a topic of current work.

The fact that simple defenses were often insufficient against adversaries led to a number of theoretical works regarding the inherent difficulty of finding classifiers that are robust to adversaries. For example, \cite{bubeck2019adversarial} suggests that in some settings computation is the primary bottleneck in constructing adversarially robust classifiers. \cite{gilmer2018adversarial,mahloujifar2019curse,shafahi2018adversarial} all highlight how high dimensional geometry induces inherent limitations in the ability to avoid adversarial examples. \cite{ilyas2019adversarial} argues that adversarial examples are often based upon human derived notions of similarity that are incompatible with the geometry and training that occurs in deep learning. \blue Finally, the interplay between the geometry of the types of perturbations used in measuring adversarial attacks was explored in \cite{khoury2018geometry}. That work demonstrated that adversarial robustness with respect to $\ell^\infty$ norm perturbations is not equivalent to $\ell^2$ norm perturbations, and that under a manifold hypothesis adversarial examples may be a consequence of the the complicated nature of high-dimensional geometry. \nc

While the above works highlight the difficulty of completely avoiding adversarial examples, they do not study the ability of classifiers to mitigate the effects of adversarial examples. One such framework for mitigating, on average, these effects is the optimal adversarial classification problem that we study here. Several variants of this problem have been previously studied. One variant permits the adversary to perturb the distribution of $(x,y)$'s that are inputted \cite{blanchet2019robust,gao2017wasserstein}; in \cite{blanchet2019robust} a family of robust regression and classification problems are seen to be equivalent to a series of regularized risk minimization problems. A second variant, considered in both \cite{bhagoji2019lower,pydi2019adversarial}, studies the data perturbing adversary. In particular, those works derive a duality principle relating the optimal classification problem for balanced classes to a optimal coupling or transportation problem. \cite{pydi2019adversarial} uses Strassen's theorem from the theory of optimal transportation \cite{villani2003topics} to derive a duality principle, and demonstrates that minimizers of the adversarial problem may be taken to be closed sets. This may be seen as an initial step towards proving that optimal adversarial classifiers are indeed smoother than ones without adversaries. \blue These works have focused on the sufficient conditions associated with duality principles, but to our knowledge there is no work deriving the necessary conditions associated with optimal decision boundaries of adversarially robust classifiers.

Tracking the effect of a regularization parameter on optimal solutions of a statistical problem has been studied in various contexts. For example, the evolution of optimal solutions of $\ell_1$ regularized regression problems (i.e. Lasso) were studied in \cite{SqRtLasso}. More recently, in the context of parametric adversarial learning, \cite{javanmard2020precise} studied the tradeoff between accuracy and adversarial robustness as a function of ``$\e$''. In that work the optimal solutions admit direct representation formulas, and hence one can directly describe the evolution of the optimal classifier. In contrast, our work focuses on non-parametric classifiers, and to our knowledge no other works attempt to describe the evolution, in terms of a differential equation, of classification boundaries as a function of the adversarial power.
\nc

Finally, it is worth mentioning that other notions of classification robustness have been introduced in the literature \cite{Astuteness}. Similar questions to the ones explored in this paper can also be studied under the setting proposed in that work.

\subsection{Geometric flows and PDE methods in learning }

Our work also draws upon ideas from geometric evolutions, and more generally variational problems. Mean curvature flow is well-studied from a theoretical standpoint, in particular as a gradient flow of the perimeter. Desirable properties of this flow, such as comparison principles, and local regularity theorems, are available in \cite{ecker2012regularity}. High fidelity numerical approximations are also available \cite{merriman1992diffusion}. Our evolution equation is also not unrelated to non-local versions of curvature flow, which also are a topic of significant current interest \cite{chambolle2015nonlocal}.

\blue
 
In recent years, there has also been a growing interest in using the ideas and techniques from the analysis of interfacial flows, to construct new algorithms in data analysis.  These algorithms arise as iterative schemes to solve optimization problems closely related to graph-based supervised, unsupervised, and semi-supervised learning; see \cite{CalatroniSch,MBOSignedNetworks,CommunityDetection,AuctionDynamics, MBOClassifImageProcessing,Merkurjev2018ASH,VanGennipGuillen} and references within. 
\nc

\section{Problem setup}\label{sec:setup}

Let $\nu$ be a Borel probability measure on $\R^d \times \{ 0,1\}$ representing a data distribution for pairs $(x,y)$ where $x$ is a feature vector and $y$ an associated label. Let $(X,Y) \sim \nu $. We assume that the conditional distribution of $X$ given $Y=0$ takes the form $\rho_0 dx$, while the conditional distribution of $X$ given $Y=1$ equals $\rho_1 dx$, for two density functions $\rho_0, \rho_1$ that are assumed to satisfy certain regularity and non-degeneracy properties that we will make precise later on (for example see Assumptions \ref{assump:densities} for the one dimensional setting). We use $\rho dx$ to denote the marginal distribution of $X$. Notice that $\rho$ can be expressed as
\[\rho= w_0 \rho_0 + w_1 \rho_1, \] 
where $ w_0= \Prob(Y=0) $ and $w_1 = \Prob(Y=1)$. We let $$\mu(x) := \P(Y = 1 | X =x)$$ represent the conditional probability (or mean) of the label variable $Y$ given $X$. \blue Our conventional notation throughout the paper is that $\rho_i(z+w)$ is always meant to denote $\rho_i$ evaluated at $z+w$, while any multiplication of $\rho_i$ by $(z+w)$ will be denoted using $\rho_i \cdot (z+w)$. We notice that as a consequence of Bayes' theorem $\mu$ may be written using
	\[  \mu(x) = \Prob(Y=1) \cdot \frac{\rho_1(x)}{\rho(x)} =  \frac{w _1\rho_1(x)}{\rho(x)}.\]
\nc

The classical classification problem seeks to minimize the functional
\[
R(f) = \E(\ell(f(x),y)) = \int \ell(f(x),y) \,d\bnu(x,y)
\]
over some class of functions $f \in \mathcal{F}$. Usually, one is required to select $f = \mathds{1}_A$ for some Borel set $A$. Of particular importance is the case when $\ell(f(x),y) = \mathds{1}_{f(x) \not = y}$ (known as the 0-1 loss), where one may actually minimize over the class of $L^1$ functions, and where minimizers of the form $\mathds{1}_A$ always exist.  In particular, the function
\[
u_B(x) = \begin{cases} 1\quad &\text{ if } \mu(x) \geq 1/2 \\ 0 &\text{ otherwise} \end{cases}
\]
known as the \emph{Bayes classifier}, is a minimizer to the 0-1 loss problem. In short, at least from a  theoretical perspective, the optimization of the risk functional $R$ relative to 0-1 loss admits a closed form solution.

In the adversarial classification problem, one supposes an adversary that is able to modify incoming data points. In particular, in this paper we imagine that the adversary is allowed to shift any data point $x$ with label $y$ to a nearby point $g(x,y)$ so that $|x-g(x,y)| \leq \e$. Here $\veps$ is a parameter that describes the power of the adversary: the larger the value of $\veps$, the more the adversary can perturb the data. In this setting, one seeks to build a classifier that minimizes the robust risk
\[
  R_\e(f) := \blue \sup_{g : \sup_x d(g(x,y),x) \leq \e} \int \ell(f(g(x,y)),y) \,d\bnu(x,y), \nc
\]
which factors in the action of the adversary.
Notice that in the above model, the adversary can use information of a feature vector $x$ as well as of its corresponding label $y$ in order to decide on the new features for that data point.  This model has been studied previously in \cite{bhagoji2019lower,pydi2019adversarial} where  interesting connections with optimal transport problems have been established. In this paper we revisit these connections and extend them.    

In order to analyze the minimization of the above robust risk, we first must characterize the $g$ which achieves the maximum risk for a given $f= \mathds{1}_A$. We begin by defining the distance between a point and a set $A \in \M(\R^d)$ via
\[
  d(x,A) := \blue \inf_{y \in A} d(x,y), \nc
\]
where $\M(\R^d)$ denotes the Borel sets of $\R^d$. For convenience, we also define a signed distance via
\[
\tilde d_A(x) = \begin{cases} d(x,A) &\text{ if } x \notin A \\ -d(x,A^c) &\text{ if } x \in A. \end{cases}
\]
The maximization problem for the adversary admits a direct representation in terms of this signed distance. In particular, we notice that for $f = \mathds{1}_A$,  if $|\tilde d_A(x)|\leq \e$, then the adversary is free to select an arbitrary response at the point $(x,y)$ regardless of the value of $y$. On the other hand, if $|\tilde d_A(x)|>\e$ the adversary is unable to modify the label $f(x)$ by moving the inputted point by distance $\e$. This information may be encoded by rewriting our objective functional $R_\veps $ in the form:
\begin{equation*}
  R_\e(\mathds{1}_A) = \int_{\tilde d_A(x)<-\e} \ell(1,y) \,d\nu(x,y) + \int_{\tilde d_A(x) > \e} \ell(0,y) \,d\nu(x,y) + \int_{|\tilde d_A(x)| < \e} \max_{z\in\{0,1\}} \ell(z,y) \,d\nu(x,y).
\end{equation*}
We notice that when $\e = 0$ this functional reduces to the standard, non-adversarial, loss.

In order to simplify notation, we define, for any $s \in \R$, the set 
\begin{equation}\label{eqn:dilation-erosion-def}
\blue  A^s:=\{x \in \R^d : \tilde{d}_A(x) \leq s\}. \nc
\end{equation}
Furthermore, in what follows we will always consider the 0-1 loss function. In that case, we may rewrite our objective function as follows:
\begin{align*}
  R_\e(\mathds{1}_A) &=  \int_{A^{-\veps}} \nrho{0} dx +  \int_{(A^{\veps})^c} \nrho{1} dx + \int_{|\tilde d_A(x)| \leq \veps} \rho(x)dx \\
   &\blue= \int_{A^{-\veps}} \nrho{0} dx  +  \int_{(A^{\veps})^c} \nrho{1} dx + \int_{A^\e \backslash A^{-\e}} \nrho{0} + \nrho{1} dx \nc \\
  &\blue= \int_{A^{\veps}} \nrho{0} dx + \int_{(A^{-\e})^c} \nrho{1} dx \nc\\
  &=   \int_{A^{\veps}} \nrho{0} dx + w_1 -   \int_{A^{-\veps}} \nrho{1} dx,
\end{align*}
\blue where we have used the fact that $\rho_1$ is a probability distribution. \nc We are interested in the robust classification problem:
\begin{equation}
\inf_{A \in \M(\R^d)} R_\e(\mathds{1}_A).
\label{eqn:RobustClassification}
\end{equation}

\subsection{Duality principle and connection to an optimal transport problem}\label{sec:duality}

Problem \eqref{eqn:RobustClassification} admits a strong duality theorem. To illustrate, we recall previous results in \cite{bhagoji2019lower,pydi2019adversarial}. In those works, they consider $w_0 = w_1 = 1/2$, in which case the robust risk minimization problem becomes
\[
 \inf_{A \in \M(\R^d)} R_\e (\mathds{1}_A) = \frac{1}{2} \left( 1 - \sup_{A \in \M(\R^d)} \left\{    \int_{A^{-\veps}} \rho_1 dx  -    \int_{A^{\veps}} \rho_0 dx   \right\} \right).
\]
It is then shown that
\[ \sup_{A \in \M(\R^d)} \left\{    \int_{A^{-\veps}} \trho_1 dx  -    \int_{A^{\veps}} \trho_0 dx   \right\} =  \inf_{\pi \in \Gamma(\trho_1, \trho_0)} \int \mathds{1}_{\blue d(x_1,x_2)>2\veps \nc}d \pi(x_1,x_2) =: d_\e(\trho_1,\trho_0), \]
where here $\Gamma(\rho_1,\rho_0)$ denotes the set of probability measures on $\R^d \times \R^d$ with marginals $\rho_1$ and $\rho_0$ (i.e. the set of couplings or transportation plans between $\rho_1$ and $\rho_0$); the above result is closely connected to Strassen's theorem (see  Corollary 1.28 in \cite{villani2003topics}).  This result may be restated in the following way
\begin{equation}
  \inf_{A \in \M(\R^d)} R_\e(\mathds{1}_A) = \sup_{\pi \in \Gamma(\trho_1,\trho_0)} \frac{1}{2}\left( 1-  \int \mathds{1}_{\blue d(x_1,x_2)>2\veps \nc}d \pi(x_1,x_2) \right).
\label{eqn:DualityBalncedVarun}
\end{equation}

\blue
This duality principle provides a means of certifying the optimality of solutions to the functional $R_\e$, as is common in the context of convex optimization. Our later proofs establishing the global optimality of solutions that we construct using evolution equations will directly utilize this duality principle. Previous results in this vein focused only on the case with balanced classes: here we extend their results to the case of unbalanced classes. Indeed, the remainder of this section provides a direct generalization of the duality results given in \cite{bhagoji2019lower,pydi2019adversarial}.
\nc

In order to state a duality principle for more general $w_i$, it will be convenient to define the probability measure on $\R^d \times \{0,1\}$ given by
\[
\nu^S(E\times \{1\}) = \nu(E \times \{0\}),\qquad \nu^S(F \times \{0\}) = \nu(F \times \{1\}).
\]
In words, $\nu^S$ is simply the data distribution after swapping the $y$ labels. Using the measures $\nu$ and $\nu^S$, we now state a more general duality principle that applies for arbitrary $w_0, w_1$ and not just for $w_0=w_1=1/2$.

\begin{proposition}
	\label{prop:Duality}
	Let $c_{\veps}: (\R^d \times \{ 0,1\})^2 \rightarrow \R$ be the cost defined by
\[  c_{\veps}(z_1, z_2):= \mathds{1}_{\{\blue d(x_1,x_2)>2\veps \nc \} \cup \{ y_1\not =y_2\} }, \]
where we write $z_i = (x_i,y_i)$.
Then,
\[  2 \sup_{B \in \M( \R^d)}   \left\{  \int_{B^{-\veps}}  \nrho{1} dx   -  \int_{B^{\veps}} \nrho{0} dx     \right\}     -w_1+w_0=  \inf_{\pi \in \Gamma(\nu, \nu^S)} \int c_{\veps}(z_1, z_2) d\pi(z_1,z_2),  \]
which is also equal to
\[  2\sup_{A \in \M( \R^d)}   \left\{   \int_{A^{-\veps}}  \nrho{0} dx   -   \int_{A^{\veps}} \nrho{1} dx     \right\}     -w_0+w_1. \]
\end{proposition}

\begin{proof}
We follow Theorem 1.27 in \cite{villani2003topics}. First, by the Kantorovich duality theorem (see Theorem 1.3 in \cite{villani2003topics}) we have

\begin{equation}
   \sup_{\phi(z_1) + \psi(z_2) \leq c_\veps(z_1, z_2)} \int \phi(z_1) d\nu(z_1) + \int \psi (z_2)d \nu^S(z_2)  = \inf_{\pi \in \Gamma(\nu, \nu^S)} \int c_\veps(z_1, z_2) d\pi(z_1, z_2).
   \label{eqn:Duality} 
\end{equation}
where the sup is over all $\phi\in L^1(\nu)$ and $\psi \in L^1(\mu)$ (known as Kantorovich potentials), and the inequality constraint must be interpreted for $\nu$ almost every  $z_1$ and for $\nu^S$ almost every $z_2$.

Let $\phi$ and $\psi$ be two arbitrary Kantorovich potentials. Notice that if $\phi(z) + \psi(\tilde z) \leq c_{\veps}(z,\tilde z)$ then necessarily  $\phi$  is (essentially) bounded above. By subtracting a constant from $\phi$  and adding this same constant to $\psi$, we can assume without the loss of generality that $\sup_{z} \phi(z) =1$. Now, for a given such $\phi$ the best corresponding $\psi$, i.e. its dual conjugate potential, is given by 
\[ \phi^{c_\veps}(\tilde z):= \inf_{z } \left\{ c_\veps(z, \tilde z) - \phi( z) \right\}. \]
Notice that $\phi^{c_\veps}$ can be written as:\blue
\begin{align*}  \phi^{c_\veps} (\tilde x,0) &= \min \begin{cases}  \inf \{ c_\veps(z, \tilde z) - \phi( z) :  z = (x,0), d(x,\bar x >2\e \} \\ \inf \{ c_\veps(z, \tilde z) - \phi( z): z = (x,0), d(x,\bar x) \leq 2\e\} \\ \inf \{  c_\veps(z, \tilde z) - \phi( z): z = (x,1) \} \end{cases} \\
&= \min \left\{ 1-   \sup_{ x: \blue d(\tilde x,x) >2\veps \nc} \phi( x,0)  ,    -\sup_{ x: \blue d(\tilde x, x)  \leq 2\veps \nc} \phi( x,0) ,   1- \sup_{ x} \phi( x,1)    \right\} , \end{align*}
Similarly we find that \nc
\[  \phi^{c_\veps} (\tilde x,1)=  \min \left\{ 1-   \sup_{ x: \blue d(\tilde x,x) >2\veps \nc} \phi( x,1)  ,    -\sup_{ x: \blue d(\tilde x,x) \leq 2\veps \nc} \phi( x,1) ,   1- \sup_{ x} \phi( x,0)    \right\} . \]
Since we have assumed that $\sup_{ z} \phi( z) = 1$ we can deduce from the above that $\phi^{c_\veps}( \tilde z) \in [-1,0]$. In particular, the supremum in \eqref{eqn:Duality} can be restricted to pairs $\phi, \psi$ satisfying the cost constraint and  $\psi \in [-1,0]$.  

Let us now consider an arbitrary $\psi$ \blue taking values in $[-1,0]$ \nc with its best associated $\phi$: 
\[  \psi^{c_\veps} (x,0):=  \min \left\{ 1-   \sup_{\tilde x: \blue d(\tilde x,x) >2\veps \nc} \psi(\tilde x,0)  ,    -\sup_{\tilde x: \blue d(\tilde x,x) \leq 2\veps \nc} \psi(\tilde x,0) ,   1- \sup_{\tilde x} \psi(\tilde x,1)    \right\} . \]
\[  \psi^{c_\veps} (x,1):=  \min \left\{ 1-   \sup_{\tilde x: \blue d(\tilde x,x) >2\veps \nc} \psi(\tilde x,1)  ,    -\sup_{\tilde x: \blue d(\tilde x,x) \leq 2\veps \nc} \psi(\tilde x,1) ,   1- \sup_{\tilde x} \psi(\tilde x,0)    \right\} . \]
Since \blue we are only considering $\psi$ which take non-positive values, \nc it follows that
\[ \psi^{c_\veps}(x,0)= - \sup_{\tilde x: \blue d(\tilde x,x) \leq 2\veps \nc} \psi(\tilde x , 0), \quad  \psi^{c_\veps}(x,1)= - \sup_{\tilde x: \blue d(\tilde x,x) \leq 2\veps \nc} \psi(\tilde x , 1),\]
which in particular implies that $\psi^{c_\veps}\in [0,1]$.  Finally, computing the conjugate of  $\phi :=\psi^{c_\veps} $ we get 
\[ \phi^{c_\veps}(\tilde x,0)= - \sup_{ x: \blue d(\tilde x,x) \leq 2\veps \nc} \phi( x , 0), \quad  \phi^{c_\veps}(\tilde x,1)= - \sup_{ x: \blue d(\tilde x,x) \leq 2\veps \nc} \phi( x , 1) \]
which is then seen to take values on $[-1,0]$. Since $\phi^{c_\veps}$ is the best $\psi$ for a given $\phi\in [0,1]$, it follows that the supremum in \eqref{eqn:Duality} is equal to
\[ \sup_{\phi \in [0,1]}   \int \phi(z) d\nu(z) + \int \phi^{c_\veps}(\tilde z) d\nu^S(z).     \]

From the fact that for arbitrary $\phi \in [0,1]$ we  have $\phi^{c_\veps}  \in [-1,0]$, we deduce, using the ``layer cake'' representation \blue (which we recall in Lemma \ref{lem:layer-cake} in the Appendix),\nc
\[ \int \phi(z) d\nu(z) + \int \phi^{c_\veps}(\tilde z) d\nu^S(\tilde z) = \int_{0}^1 \int \mathds{1}_{\phi(z) > s } d\nu(z) ds  -   \int_{0}^1 \int \mathds{1}_{-\phi^{c_\veps}(\tilde z) > s } d\nu^S(\tilde z) ds ,\]
\begin{equation}
\ =  \int_{0}^1 \left(  \int \mathds{1}_{\phi(z) > s } d\nu(z) ds  -    \int \mathds{1}_{-\phi^{c_\veps}(\tilde z) > s } d\nu^S(\tilde z)    \right) ds.
\label{eqn:Dualityaux}  
\end{equation}
We now rewrite the indicator function $\mathds{1}_{-\phi^{c_\veps}(\tilde z) \geq s}$ in terms of a $2\veps$-expansion of a set. Indeed, for $\tilde z=(\tilde x,0)$ we have:
\begin{align*}
\mathds{1}_{\{-\phi^{c_\veps}(\cdot) > s\}}(\tilde z)=1 & \Leftrightarrow -\phi^{c_\veps}(\tilde x,0) > s
\\& \Leftrightarrow \exists  x \text{ s.t. } \blue d(x,\tilde  x)\leq 2\veps  \nc \text{ and }   \phi(x, 0)> s
\\& \Leftrightarrow \tilde x\in \{  x \: : \:  \phi(x, 0) >s \}^{2\veps}.
\end{align*}
Thus, $ \mathds{1}_{\{ -\phi^{c_\veps}(\cdot) > s \}}(\tilde x,0) = \mathds{1}_{ \{  \phi(\cdot , 0) >s \}^{2\veps}}(\tilde x)$. In the exact same way we see that $\mathds{1}_{\{-\phi^{c_\veps}(\cdot)> s \}}(\tilde x,1) = \mathds{1}_{ \{  \phi(\cdot, 1) >s \}^{2\veps}}(\tilde x)$. Since we are integrating over $s \in [0,1]$, we may infer that there exists $s \in [0,1]$ such that
\[ \int_{0}^1 \left(  \int \mathds{1}_{\phi(z) > s } d\nu(z) ds  -    \int \mathds{1}_{-\phi^{c_\veps}(\tilde z) > s } d\nu^S(\tilde z)    \right) ds  \leq  \int \mathds{1}_{\phi(z) > s } d\nu(z) ds  -    \int \mathds{1}_{-\phi^{c_\veps}(\tilde z) > s } d\nu^S(\tilde z)  \] 
\begin{align*}
\begin{split}
 &= \int \mathds{1}_{\{\phi(x,0) > s\}} \nrho{0}(x) dx + \int \mathds{1}_{\{\phi(x,1) > s\}} \nrho{1}(x) dx 
 \\& - \int \mathds{1}_{\{\phi(x,0) > s\}^{2\veps}} \nrho{1}(x) dx - \int \mathds{1}_{\{\phi(x,1) > s\}^{2\veps}} \nrho{0}(x) dx,
 \end{split}
 \end{align*}
 where we have used the definitions of $\nu$ and $\nu^S$.
%
The above computations,  along with \eqref{eqn:Dualityaux}, \nc allow us to conclude that:
\begin{align*}
 &\inf_{\pi \in \Gamma(\nu, \nu^S)} \int c_{\veps}(z_1, z_2) d\pi(z_1,z_2) \\
 &= \sup_{A\in \M(\R^d)}   \left\{  \int_{A}  \nrho{0} dx   -   \int_{A^{2\veps}} \nrho{1} dx      \right\}  + \sup_{B\in \M(\R^d)}  \left\{   \int_{B}\nrho{1} dx -  \int_{B^{2\veps}} \nrho{0} dx  \right\} \\
 &=   \sup_{A\in \M(\R^d)}   \left\{   \int_{A^{-\veps}}  \nrho{0} dx   -  \int_{A^{\veps}} \nrho{1} dx     \right\}  + \sup_{B \in \M(\R^d)}  \left\{   \int_{B^{-\veps}}\nrho{1} dx -  \int_{B^{\veps}} \nrho{0} dx  \right\}\\
 &=   \sup_{A \in \M(\R^d)}   \left\{   \int_{\blue (A^{c})^{-\veps} \nc}  \nrho{1} dx   -  \int_{\blue (A^{c})^{\veps} \nc} \nrho{0} dx      \right\}  
 \\& \quad  + \sup_{B \in \M(\R^d)}  \left\{  \int_{B^{ -\veps}} \nrho{1} dx -   \int_{B^{\veps}}\nrho{0} dx   \right\} -w_1 +w_0\\
&= 2  \sup_{B \in \M(\R^d)}   \left\{  \int_{B^{-\veps}}  \nrho{1} dx   - \int_{B^{\veps}} \nrho{0} dx     \right\}  -w_1 +w_0.\end{align*}
\blue In the previous computation the step from line two to line three is the only one which does not follow directly from definitions: its justification relies upon technical measure-theoretical arguments which can be found in the appendix of \cite{pydi2019adversarial}, and which we omit here for the sake of brevity. \nc
Notice that we  also obtain:
\[ = 2  \sup_{A \in \M(\R^d)}   \left\{   \int_{A^{-\veps}}  \nrho{0} dx   -   \int_{A^{\veps}} \nrho{1} dx     \right\}  -w_0 +w_1.\]

This shows our desired result.

\end{proof}

\blue
We now translate the previous proposition, which mirrors the terminology used in describing duality in optimal transportation and linear programming, into a form which directly links the adversarial classification problem with the transportation problem from the previous proposition.
\nc

\begin{corollary}
	\label{cor:Risk}
$\mathds{1}_A$ for some $A \in \M(\R^d)$ minimizes $R_{\veps}$ if and only if $A$ maximizes
\[\sup_{A \in \M(\R^d)} \left\{ w_1 \int_{A^{-\veps}} \rho_1 dx - w_0 \int_{A^{\veps}} \rho_0 dx \right\}.\]
	
Moreover, 	
	\[\inf_{A \in \M(\R^d)} {R}_\veps(\mathds{1}_A) =  \frac{1}{2} - \frac{1}{2} \inf_{\pi \in \Gamma(\nu, \nu^S)} c_\veps(z_1, z_2) d\pi(z_1,z_2). \]
\end{corollary}
\begin{proof}
Recall that
\[R_\veps(\mathds{1}_A) =  \int_{A^{\veps}} \nrho{0} dx + w_1 -   \int_{A^{-\veps}} \nrho{1} dx \]
so 
\[  \inf_{A \in \M(\R^d)} R_\veps(\mathds{1}_A) = w_1  - \sup_A \left\{ \int_{A^{-\veps}} \nrho{1} dx -  \int_{A^{\veps}} \nrho{0} dx \right\} \]
\[ = w_1 - \frac{1}{2}(w_1- w_0) - \frac{1}{2} \inf_{\pi \in \Gamma(\nu, \nu^S)} c_\veps(z_1, z_2) d\pi(z_1,z_2) \]
\end{proof}

\begin{remark}
Let us consider the balanced case $w_0=w_1=1/2$. Since
\[ \inf_{A\in \M(\R^d)} R_\veps (A) = \frac{1}{2} \left(1 -  \inf_{\pi \in \Gamma(\nu, \nu^S)} c_\veps(z_1, z_2) d\pi(z_1,z_2) \right), \]
it follows that
\[   \inf_{\pi \in \Gamma(\nu, \nu^S)} \int c_{\veps}(z_1, z_2) d\pi(z_1,z_2)=    \inf_{\gamma \in \Gamma(\rho_0, \rho_1)} \int \mathds{1}_{\blue d(x_1,x_2)>2\veps \nc} d\gamma(x_1,x_2). \]

\end{remark}

\nc

\section{Necessary conditions and corresponding evolution equation in one dimension}\label{sec-1D}

We now describe, in detail, the necessary conditions for minimizing $R_\e$, and the evolution equation that they induce. For clarity, we begin by describing this evolution equation in the simple case where $x \in \R$, \blue under the standard metric. \nc In this case we will be able to prove that the resulting evolution equation completely characterizes the global minimizer of $R_\e$ for small $\e$ under mild assumptions; the formal statement and proof of this result is given in Section \ref{sec:Global}. Subsequently, in Section \ref{sec:high-D} we will turn our attention to the case where $x \in \R^d$.

To begin, let us assume that we may represent the boundary of the optimal set $A_\e^*$ in terms of two \blue parametrized collections of points $a_i(\e)$ and $b_i(\e)$, so that $A_\e^* = \cup_{i=1}^K [a_i(\e),b_i(\e)]$. \nc Here we allow $a_1 = -\infty$ and $b_K = +\infty$ if necessary, and we notice that, as $\nrho{0},\nrho{1}$ are both absolutely continuous (see Assumptions \ref{assump:densities}), it makes no difference whether the sub-intervals are open or closed. We note that this assumption will hold for $\e=0$ as long as the set where $\nrho{0} = \nrho{1}$ is a discrete set, a mild assumption. Finally,  in the remainder we may suppress the dependence of $a_i,b_i$ on $\e$, in order to decrease the notational burden. \blue Furthermore, and following our notational convention for the $\rho_i$, any time $a_i,b_i$ are followed by parentheses we always mean that the parentheses denote function evaluation: cases where multiplication is implied will be denoted by $a_i \cdot z$. \nc We use the convention $a_1 <b_1 <a_2< b_2<\dots< a_K <b_K $.


As $A_\e^*$ is a minimizer of $R_\e$, we may freely perturb the boundary points (i.e. $a_i, b_i$) without increasing the energy. In particular, for $|\delta|$ small enough, if we consider the set $A(\delta) = (a_1,b_1+\delta) \cup \left( \cup_{i=2}^K (a_i,b_i) \right)$, then since $A_\e^*$ is a minimizer we have that $R_\e(A(\delta)) - R_\e(A_\e^*) \geq 0$.  
Taking $\delta \to 0$ and using the fundamental theorem of Calculus then allows us to write
\begin{align*}
0 &= \lim_{\delta \to 0} \frac{R_\e(A(\delta))-R_\e(A_\e^*)}{\delta}\\
&= \nrho{0}(b_1 + \e) - \nrho{1}(b_1-\e). 
\end{align*}
An analogous argument for the $a_i$ and for the rest of the $b_i$ then allows us to write the necessary conditions:
\begin{equation}\label{eqn:1D-Necessary}
 \nrho{1}(b_i-\veps) =  \nrho{0}(b_i+\veps), \quad \nrho{1}(a_i+\veps) = \nrho{0}(a_i-\veps),
\end{equation}
which hold for all $a_i$ and $b_i$ that are not $-\infty$ or $+\infty$. In the remainder, if $a_1(0)=-\infty$ we set $a_1(\veps)=-\infty$ for $\veps>0$ and likewise if $b_{K}(0)=+\infty $, then $b_K(\veps)=+\infty$. This relates to the fact that our differential equation approach does not track potential ``topological changes'' in the decision boundaries, and is mostly focused on the case where $\e$ is small. 
%
%
%
%
%
We remark that when $\e = 0$, the above necessary condition gives precisely $\nrho{0} = \nrho{1}$, which characterizes the boundary points of the Bayes classifier. In a sense, we may view the necessary condition above as a \emph{non-local algebraic} condition: namely, that the condition that $\nrho{0}(b_i) = \nrho{1}(b_i)$ (for $\e =0$) has been replaced by the non-local algebraic condition $\nrho{0}(b_i+\e) = \nrho{1}(b_i - \e)$ (for $\e>0$).
%
%

Using the necessary conditions \eqref{eqn:1D-Necessary}, we can exactly describe the local evolution of the boundary of the set $A_\veps^*$ for small changes in $\e$. In particular, let us suppose that each boundary point varies smoothly in $\e$, namely that we express $a_i(\e)$ and $b_i(\e)$ as smooth functions in $\e$. Differentiating the necessary condition and using the chain rule, we find that
\[
   \nrho{0}'(b_i+\e) \left(\frac{db_i}{d\e} + 1\right) -   \nrho{1}'(b_i-\e) \left(\frac{db_i}{d\e} - 1\right) = 0.
\]
%
We may then solve this equation for $\frac{db_i}{d\e}$,
\begin{equation}\label{eqn:evolutiod-1d}
  \frac{db_i}{d\e} = - \frac{\nrho{0}'(b_i+\e) + \nrho{1}'(b_i-\e)}{\nrho{0}'(b_i+\e) - \nrho{1}'(b_i-\e)}.
\end{equation}
The necessary condition for the $a_i$ is analogous:
\begin{equation}\label{eqn:evolutiod-1db}
  \frac{da_i}{d\e} = - \frac{\nrho{1}'(a_i+\e) +\nrho{0}'(a_i-\e)}{\nrho{1}'(a_i+\e) - \nrho{0}'(a_i-\e)}.
\end{equation}
We continue to use the convention that \blue $a_1(\e) =-\infty$ when $a_1(0)=-\infty$ and similarly  $b_K(\e) =+\infty$ if $b_K(0)=+\infty$\nc.

The previous equations allow us to precisely describe (locally) the evolution of the decision boundaries using ordinary differential equations. In particular, beginning at $\e=0$ with the decision boundary of the Bayes classifier, we may directly solve for the optimizer of $R_\e$ by solving a system of at most $2K$ decoupled differential equations. High fidelity approximations of these equations may be obtained using standard software packages.

We remark that the differential equations at $\e=0$ are much simpler, for example: 
\begin{equation}
  \frac{db_i}{d\e}[\e=0] = -\frac{\rho'(b_i)}{\nrho{0}'(b_i) - \nrho{1}'(b_i)}.
    \label{eqn:DerivativesatZero}
\end{equation}
This indicates that the $b_i$ initially moves downhill in $\rho$, with speed dictated by the inverse of the derivative of the difference between the probability of the different classes. To determine the sign of the denominator, we notice that since $w_1\rho_1$ is assumed to be larger than $w_0\rho_0$ inside $(a_i(0),b_i(0))$, it is natural to assume that $w_1\rho_1'(b_i(0)) < w_0 \rho_0'(b_i(0))$. This assumption is made explicit in Assumption \ref{assump:densities}). A similar conclusion holds for the left endpoints $a_i$. Although the above non-local formulas are not too complicated here, the analogous approximation near $\e=0$ will be more important in understanding the geometric flow induced in dimension higher than one as we will see in Section \ref{sec:high-D}.

\nc

\subsection{Simple example} \label{sec:1d-example}

Suppose that $\Prob(X\in dx| Y=1) = \phi(x)dx$, where $\phi$ is the standard normal density $\phi(x) = \frac{1}{\sqrt{2\pi}}\exp(-x^2/2)$, and let $\Prob(X \in dx | Y=0) = \frac{\phi((x-2)/2)dx}{2}$.  Assume also that $\Prob(Y=1)=\Prob(Y=0)=1/2$. Since the variances of the two Gaussians $\Prob(X\in dx|Y=0)$ and $\Prob(X=x|Y=1)$ are different, their densities must intersect at exactly two points, in this case at 
\[a_1(0)= -\frac{2}{3}\left( 1+ \sqrt{2(2 +3 \log(2))} \right) \approx -2.57 , \quad  b_1(0) =  \frac{2}{3}\left( \sqrt{2(2 +3 \log(2))}-1  \right) \approx 1.23. \]
The corresponding Bayes classifier for this problem is the indicator function of the set $(a_1(0),b_1(0))$.

Since the solutions $a_1$ and $ b_1$ of the ODEs  \eqref{eqn:evolutiod-1d} and \eqref{eqn:evolutiod-1db} satisfy the necessary conditions \eqref{eqn:1D-Necessary}, it follows from Theorem 2 in \cite{pydi2019adversarial} (which characterizes optimality in the Gaussian setting) that the set $A^*_\veps:= (a_1(\e),b_1(\e))$ is a global solution to the adversarial robust problem \eqref{eqn:RobustClassification} for all $\veps$ small enough.

In order to provide concrete numerical values for the decision boundary as a function of $\veps$ we use a standard ODE solver in Python. The decision boundary, as well as the associated densities, are given in Figure \ref{fig:two-normal}. \blue We notice that $a_\e$ moves to the left, which is consistent with the fact that $\rho$ has positive slope at $a(0)$. Similarly, the $b_\e$ moves to the right, which is consistent with the fact that $\rho$ has negative slope at $b(0)$. These decision boundaries required \textit{no optimization}, and are provably global minimizers for small $\e$. \nc

\begin{figure}[h!]
  \centering
  \includegraphics[width=.4\textwidth]{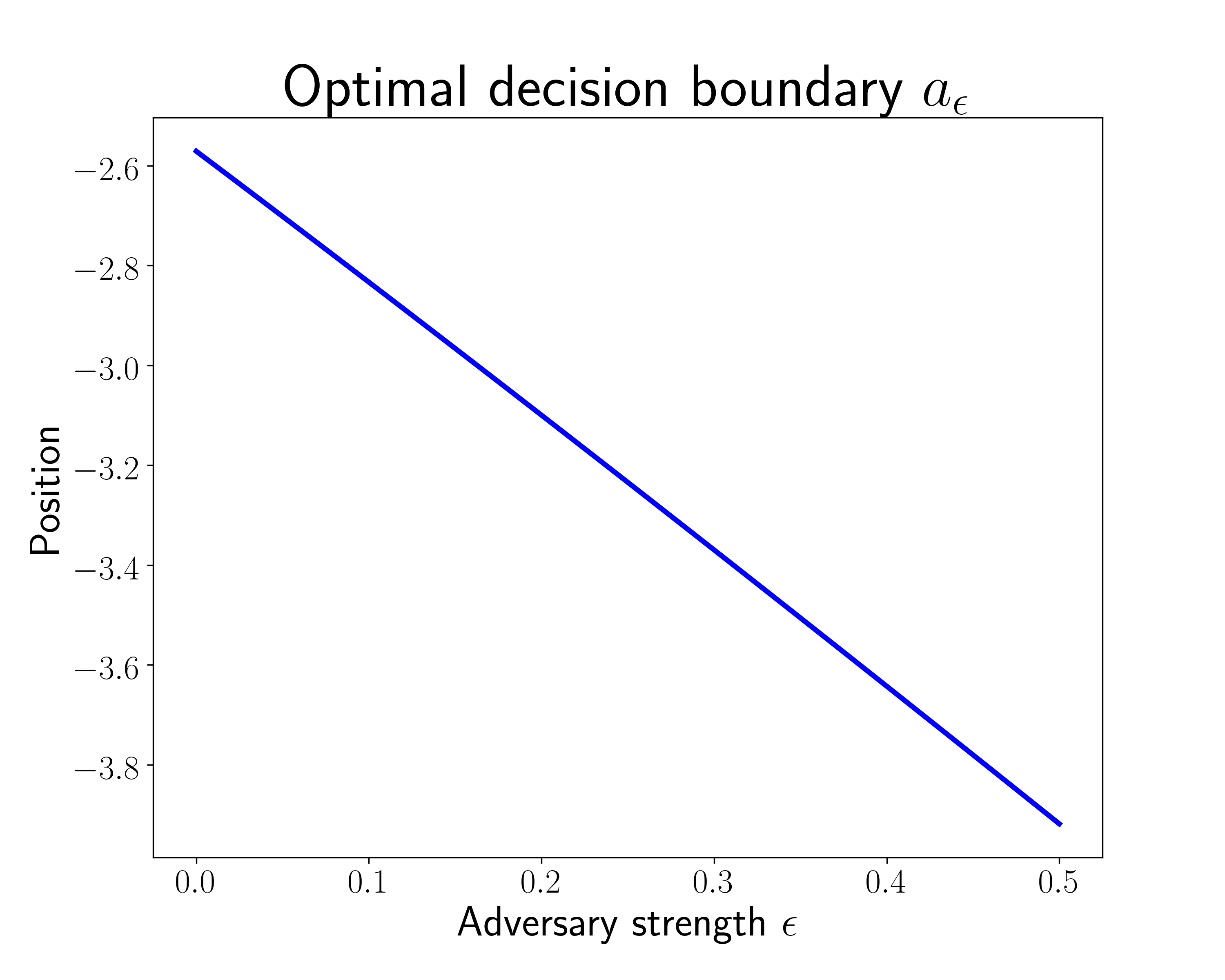}
  \includegraphics[width=.4\textwidth]{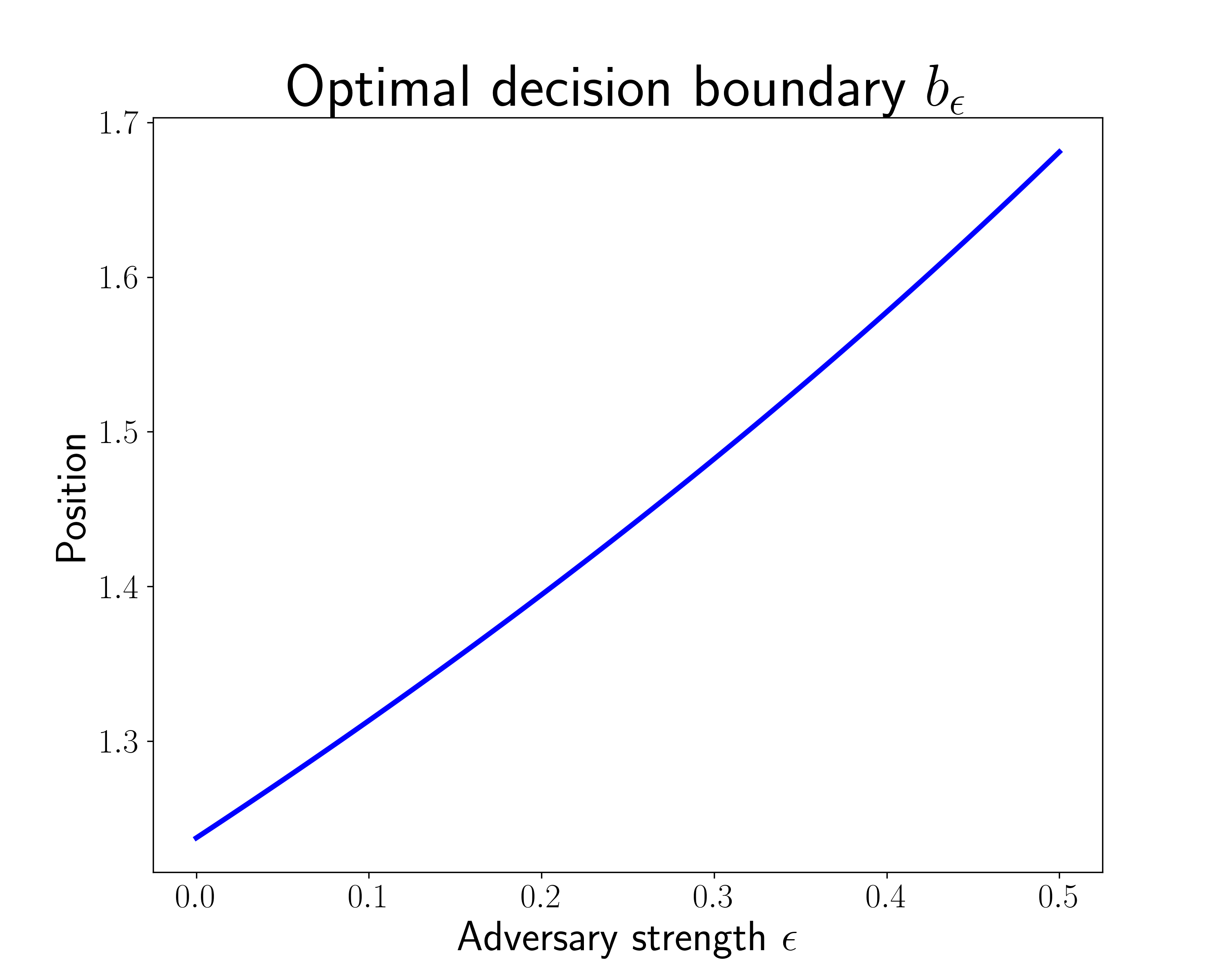}
  \includegraphics[width=.4\textwidth]{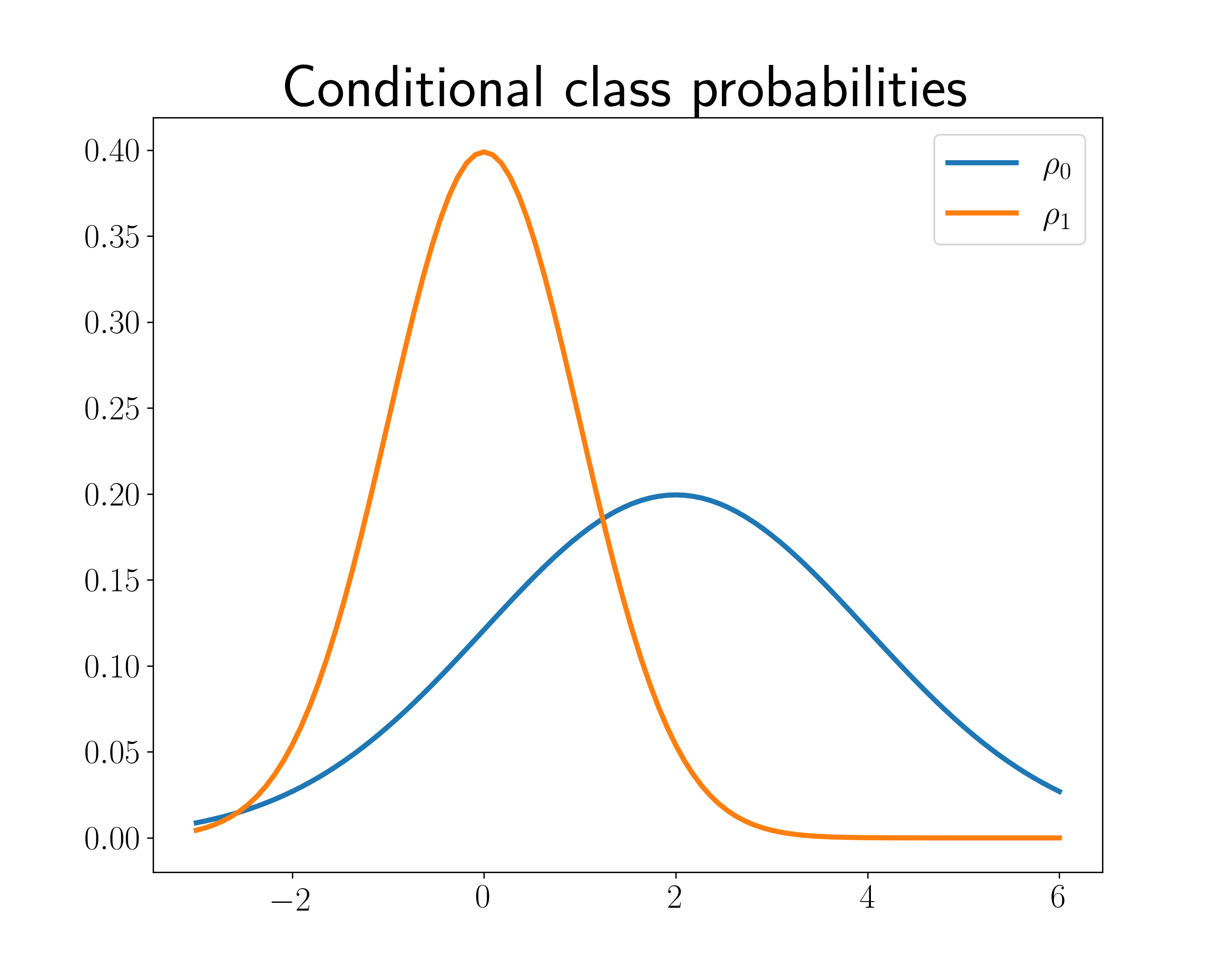}
  \caption{Plot of decision boundaries as $\e$ varies for example in Section \ref{sec:1d-example} , as well as the underlying probabilities.}
    \label{fig:two-normal}
\end{figure}

\section{Global minimizers in one dimension}\label{sec:Global}

The evolution equations of the previous section are based upon necessary conditions for the adversarial classification problem. Since they are based upon necessary conditions, it is not immediately obvious whether or not these solutions are global minimizers of the adversarial variational problem \eqref{eqn:RobustClassification}. The goal of this section is to prove that solutions of the evolution equation are indeed global minimizers for all small enough $\e$, or in other words that the evolution equation locally characterizes the minimizers of the adversarial problem. In order to do so, we will require the following mild assumptions on the densities $\nrho{0},\nrho{1}$:

\begin{assumption}
	\label{assump:densities}
	We make the following assumptions on the densities $\rho_{0}$ and $\rho_{1}$. 
	\begin{enumerate}
		\item Regularity condition: $\rho_{0}, \rho_{1} \in C^1(\R)$.
		\item Non-degeneracy condition I: there are only finitely many $t \in \R$ for which $\nrho{0}(t)= \nrho{1}(t)>0$.
		\item Non-degeneracy condition II: for every $t \in \R$ for which $\nrho{0}(t ) = \nrho{1}(t)>0$ we have $w_0\rho'_0(t ) \not = w_1\rho'_1(t)$.
	\end{enumerate}
\end{assumption}

\blue
We pause to briefly discuss these assumptions. First, we notice that the points in the boundary of the Bayes' classifier will necessarily satisfy $w_0\rho_0(t) = w_1\rho_1(t)$. Condition ii) then can be restated as requiring that the Bayes classifier be composed of finitely many intervals on the support of $\rho$. Condition iii), which only makes sense if we assume enough regularity, i.e. Condition i), rules out degeneracies, implying that the Bayes' classifier is essentially unique. Condition iii) also implies that the Bayes' classifier is stable under $C^1$ perturbations of the $\rho_i$: in a sense, this is what is leveraged in the proof of our main result, namely Theorem \ref{thm:GlobalOptimality}. These conditions should be seen as relatively mild, and indeed Condition iii) ought to be generic within the class of $C^1$ functions. For example, if we require that the $\rho_i$ be given by finite mixtures of Gaussians, then these assumptions will hold for almost every choice of parameters (i.e. means and variances). Before stating the main result, \nc we begin with a few remarks which will be important in our proof strategy.

%

\red 
\nc


\begin{remark}[Global optimality via duality]
Suppose that $A$ is a measurable subset of $\R^d$ that satisfies 
\begin{equation}
\frac{1}{2}\int c_\veps(z_1, z_2) d\pi_\veps(z_1, z_2) + \frac{1}{2}(w_1-w_0) \leq  \int_{A^{-\veps}} w_1 \rho_1 dx - \int_{A^{\veps}} w_0 \rho_0 dx,
\label{eqn:Aux}  
\end{equation}
for some $\pi_\veps \in \Gamma(\nu, \nu^S)$. Then, it follows from Proposition \ref{prop:Duality} that $A$ and $\pi_\veps $ are solutions to the optimization problems in that same proposition, and by Corollary \ref{cor:Risk}, $A$ is also a minimizer of \eqref{eqn:RobustClassification}.

\label{rem:Duality} 
\end{remark}

\begin{remark}[Knott-Smith optimality criterion]
According to Remark \eqref{rem:Duality}, to show that a given measurable set $A$ is an optimizer for \eqref{eqn:RobustClassification}, we would need to construct a coupling $\pi_\veps$ for which \eqref{eqn:Aux} holds. Now, let $A$ be a measurable subset of $\R^d$, and suppose that $\pi_\veps \in \Gamma(\nu, \nu^S)$ is concentrated on the set:
\begin{equation}
\{ (z_1, z_2) \in  (\R^d \times\{ 0,1\} )^2 \: : \:  \mathds{1}_{A^{-\veps} \times \{ 0\}}(z_1)  - \mathds{1}_{A^{\veps} \times \{0\}}(z_2) +  \mathds{1}_{(A^{c})^{-\veps} \times \{1\}}(z_1) - \mathds{1}_{(A^{c})^{\veps} \times \{1\}}(z_2)
= c_{\veps}(z_1, z_2)    \}.
\end{equation}
Then, it is straightforward to check that $A$ and $\pi_\veps$ satisfy \eqref{eqn:Aux} (with equality). The above condition for $\pi_\veps$ suggests then how mass must be exchanged between the measures $\nu$ and $\nu^S$ in order to get an optimal coupling. This insight is used to build the coupling from Theorem \ref{thm:GlobalOptimality} below.
\end{remark}

We are now ready to state the main result of this section, which states that the evolution equations \eqref{eqn:evolutiod-1d} and \eqref{eqn:evolutiod-1db} locally characterize minimizers of the adversarial classification problem.

\begin{theorem}
	\label{thm:GlobalOptimality}
	Under Assumptions \ref{assump:densities}  on $\rho_0, \rho_1, w_0, w_1$, there exists $\veps_0>0$ such that for every $\veps \in [0, \veps_0]$ there exists a coupling $\pi_\veps \in \Gamma(\nu, \nu^S)$ satisfying:
	\[   w_1 \int_{(A^*)^{-\veps}}  \rho_1 dx   -  w_0 \int_{(A^*)^{\veps}} \rho_0dx        = \frac{1}{2}  \int c_{\veps}(z_1, z_2) d\pi_\veps(z_1,z_2) +\frac{1}{2}(w_1-w_0),  \] 
	where $A^*=A^*_\veps:= \bigcup_{i=1}^K (a_i(\veps), b_i(\veps))$ and the functions $a_i, b_i$ solve the Equations \eqref{eqn:evolutiod-1d} and \eqref{eqn:evolutiod-1db}, \blue which we recall to be
	\begin{equation}
	    \frac{db_i}{d\e} = - \frac{\nrho{0}'(b_i+\e) + \nrho{1}'(b_i-\e)}{\nrho{0}'(b_i+\e) - \nrho{1}'(b_i-\e)},\tag{\ref{eqn:evolutiod-1d} revisited}
	\end{equation}
	\begin{equation}
	     \frac{da_i}{d\e} = - \frac{\nrho{1}'(a_i+\e) +\nrho{0}'(a_i-\e)}{\nrho{1}'(a_i+\e) - \nrho{0}'(a_i-\e)},\tag{\ref{eqn:evolutiod-1db} revisited}
	\end{equation} \nc
	with initial conditions $a_i(0), b_i(0)$; here, the points $a_i(0),b_i(0)$ form the decision boundary for the Bayes classifier.  In particular, according to Remark \ref{rem:Duality} the set $A^*_\veps$ induces an optimal robust classifier for $\veps$, i.e. it is a solution to the Problem \eqref{eqn:RobustClassification}.
\end{theorem}

\blue
The proof of this result is somewhat technical, because one needs to explicitly construct the transportation plans in question. We construct these plans in several steps, which can be related to the different cases in the 0-1 cost function. The most crucial step of this construction, and the one which requires Assumption iii), involves how to construct the part of the transportation plan close to the classification boundary. This is carried out in Steps 1-3 below, and is illustrated in Figure \ref{fig:OTPlan}. Step 4 constructs the (relatively simple) remainder of the transportation plan and wraps up the proof.
\nc
\begin{proof}

	Let us recall that by convention we have ordered the endpoints as $a_1(0)< b_1(0) < a_2(0 )< b_2(0) < \dots < a_K(0) < b_K(0)$. We can pick $\delta>0$ small enough so that for all $i>1$ we have 
	\[   b_{i-1} (0)+ \delta <  a_i(0) - \delta < a_{i}(0) +\delta < b_{i}(0) -\delta.\] 
	For $i=1$, if $a_1(0)$ is finite the same inequality applies, interpreting $b_{0}(0):=-\infty$. Similarly, 
		\[   a_{i} (0)+ \delta <  b_i(0) - \delta < b_{i}(0) +\delta < a_{i+1}(0) -\delta\] 
	for $i<K$, and if $b_K(0)< +\infty$ the same inequality applies, interpreting $a_{K+1}(0):=+\infty$. We notice that the solutions of the evolution equations \eqref{eqn:evolutiod-1d} and \eqref{eqn:evolutiod-1db}, which will possess local solutions under Assumption \ref{assump:densities}, are guaranteed to satisfy the necessary conditions \eqref{eqn:1D-Necessary}. This fact will be used repeatedly below.\nc

		 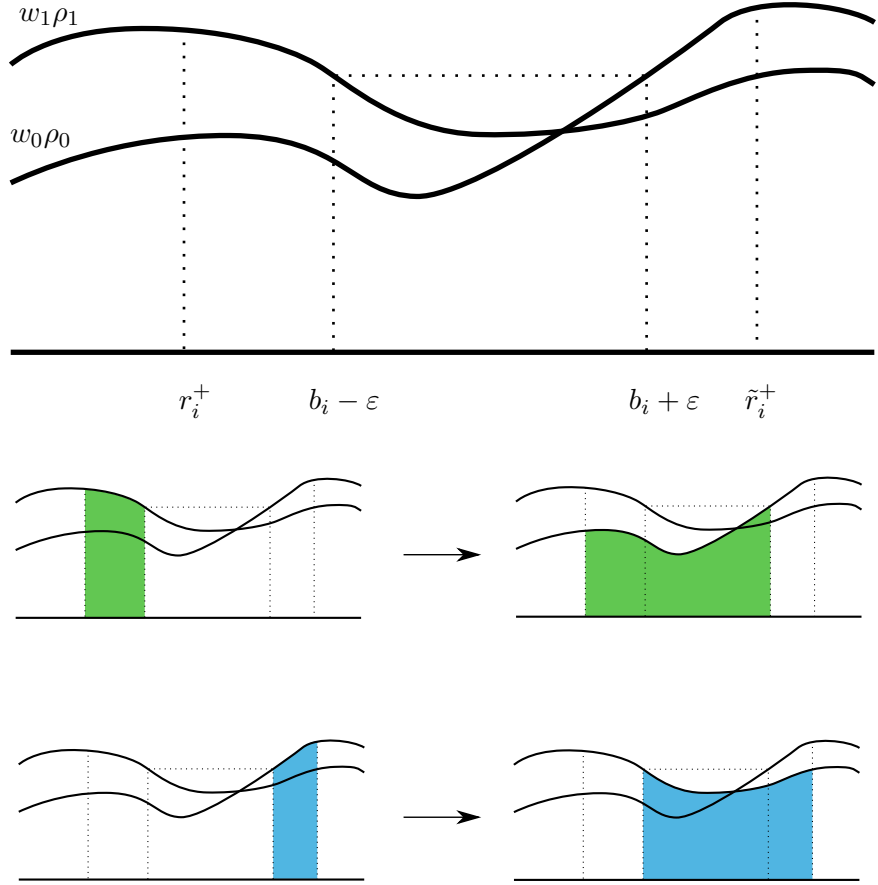
\begin{figure}[h]
 \centering
\def\svgwidth{.75\columnwidth}
\begingroup%
  \makeatletter%
  \providecommand\color[2][]{%
    \errmessage{(Inkscape) Color is used for the text in Inkscape, but the package 'color.sty' is not loaded}%
    \renewcommand\color[2][]{}%
  }%
  \providecommand\transparent[1]{%
    \errmessage{(Inkscape) Transparency is used (non-zero) for the text in Inkscape, but the package 'transparent.sty' is not loaded}%
    \renewcommand\transparent[1]{}%
  }%
  \providecommand\rotatebox[2]{#2}%
  \newcommand*\fsize{\dimexpr\f@size pt\relax}%
  \newcommand*\lineheight[1]{\fontsize{\fsize}{#1\fsize}\selectfont}%
  \ifx\svgwidth\undefined%
    \setlength{\unitlength}{482.97559363bp}%
    \ifx\svgscale\undefined%
      \relax%
    \else%
      \setlength{\unitlength}{\unitlength * \real{\svgscale}}%
    \fi%
  \else%
    \setlength{\unitlength}{\svgwidth}%
  \fi%
  \global\let\svgwidth\undefined%
  \global\let\svgscale\undefined%
  \makeatother%
  \begin{picture}(1,1.01460762)%
    \lineheight{1}%
    \setlength\tabcolsep{0pt}%
    \put(0,0){\includegraphics[width=\unitlength,page=1]{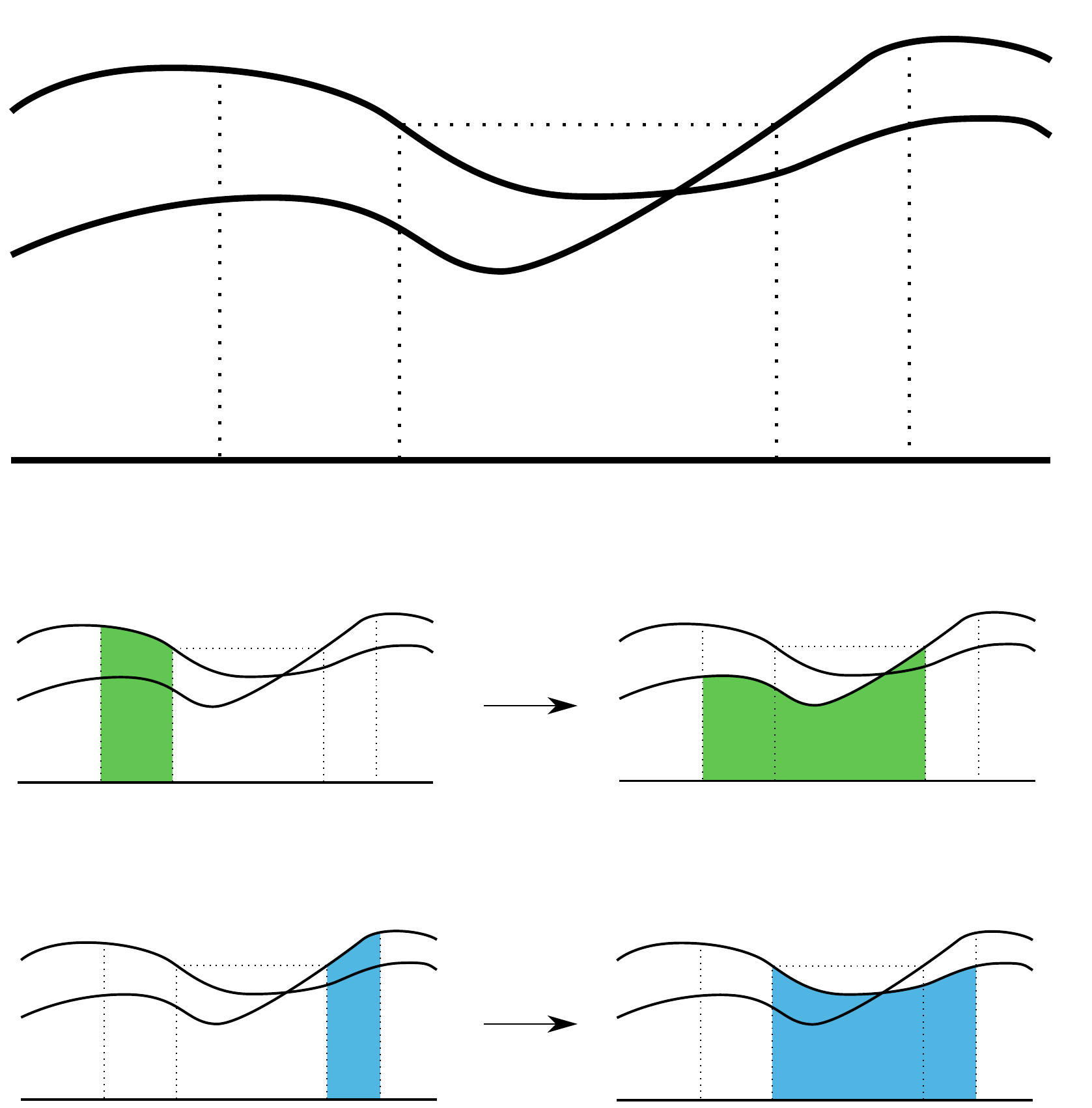}}%
    \put(0.00967104,0.82610685){\color[rgb]{0,0,0}\makebox(0,0)[lt]{\lineheight{1.25}\smash{\begin{tabular}[t]{l}$w_0 \rho_0$\end{tabular}}}}%
    \put(0.01938865,0.96093926){\color[rgb]{0,0,0}\makebox(0,0)[lt]{\lineheight{1.25}\smash{\begin{tabular}[t]{l}$w_1 \rho_1$\end{tabular}}}}%
    \put(0.19552109,0.52971849){\color[rgb]{0,0,0}\makebox(0,0)[lt]{\lineheight{1.25}\smash{\begin{tabular}[t]{l}$r_i^+$\end{tabular}}}}%
    \put(0.81927286,0.52971849){\color[rgb]{0,0,0}\makebox(0,0)[lt]{\lineheight{1.25}\smash{\begin{tabular}[t]{l}$\tilde{r}_i^+$\end{tabular}}}}%
    \put(0.69172858,0.52971849){\color[rgb]{0,0,0}\makebox(0,0)[lt]{\lineheight{1.25}\smash{\begin{tabular}[t]{l}$b_i + \varepsilon$\end{tabular}}}}%
    \put(0.33885642,0.52971849){\color[rgb]{0,0,0}\makebox(0,0)[lt]{\lineheight{1.25}\smash{\begin{tabular}[t]{l}$b_i - \varepsilon$\end{tabular}}}}%
  \end{picture}%
\endgroup%

\caption{Illustration of mass exchange defined by $\gamma_{b_i}$ (middle) and by $\tilde{\gamma}_{b_i}^{-1}$ (bottom).}
	\label{fig:OTPlan}
\end{figure}

	\textbf{Step 1: \blue [Construct matching from the left of $b_i$]} \blue
	Figure \ref{fig:OTPlan} provides a visual illustration of the particular construction that we use in our transportation plan near the boundary point $b_i(\e)$. In Step 1, we are focusing on constructing the mapping involving the green mass in the plot. This mapping is constructed in such a way that the total green mass on the left and the right are equal, and so that no mass needs to travel a distance greater than $2\e$. We notice that the heights at $b_i(\e) - \e$ and $b_i(\e) + \e$ are equal, and so that mass travels distance exactly $2\e$. The slope conditions assumed in Assumption \ref{assump:densities}, which we can show continues to hold locally near $b_i(0)$, is what allows us to infer that the rest of the green mass indeed is transported less than distance $2\e$. \nc 
	
	To begin the construction, let us fix a \blue particular $i$ corresponding to a finite right endpoint $b_i(\e)$ \nc and notice that for small enough $\veps>0$ we have
	\[ b_i(0) -\delta/2 \leq  b_i-\veps  < b_i+\veps < b_i(0)+\delta /2 <  a_{i+1}(0)- \delta , \]
	where we recall that $b_i=b_i(\veps)$ (we have dropped the dependence on $\veps$ to ease the notation).
	\nc
	Now, by Assumption \ref{assump:densities} we know that $w_0 \rho_0'(b_i(0)) \not = w_1 \rho_1'(b_i(0)) $. \blue Given that when $\e =0$ we have that $A_\e^*$ is the Bayes' classifier, we may deduce that $w_0\rho_0 <  w_1\rho_1 $ inside $(a_i(0), b_i(0))$. Hence $w_0 \rho_0'(b_i(0)) > w_1 \rho_1'(b_i(0))$. \nc Moreover, the fact that $\rho_0, \rho_1$ are $C^1(\R)$ allows us to deduce that
	\begin{equation} 
	w_0\rho_0'(t_0) > w_1 \rho_1'(t_1) 
	\label{eqn:densitiesDerivatives}
	\end{equation}
	for every $t_0,t_1 $ in  $[b_i(0)-\delta, b_i(0)+ \delta]$ (by making $\delta$ smaller if needed). In particular, for all $\veps>0$ small enough we have
\begin{equation}
	 \frac{d}{ds}  \left( w_0 \rho_0(b_i+\veps - s) \right)  <  \frac{d}{ds}  \left( w_1 \rho_1(b_i-\veps-s) \right), \quad \forall s \in(0, \delta/2 ).
	 \label{eqn:auxOPlanCons}
	\end{equation}
The above condition can be combined with the necessary condition for $b_i$ in \eqref{eqn:1D-Necessary} and the fundamental theorem of Calculus to obtain
	\begin{equation}
	w_0\rho_0( b_i+\veps-s) \leq w_1 \rho_1(b_i-\veps -s) , \quad \forall s \in (0, \delta/2),
	\label{eqn:AuxDensities1}
	\end{equation}
	for all small enough $\veps>0$.
	
	Let $r_i^+$  (which depends on $\veps$) be the largest number smaller than $b_i-\veps$ satisfying:
	\begin{equation}
	\int_{r_i^+}^{b_i-\veps} w_1 \rho_1(x)dx = \int_{r_i^+}^{b_i+\veps}w_0\rho_0(x) dx. 
	\label{eqn:BalancingCondition}
	\end{equation}
	 The existence of $r_i^+$ (at least for small enough $\veps$) follows from \eqref{eqn:auxOPlanCons} and condition \eqref{eqn:1D-Necessary}, which combined also imply that $r_i^+$  satisfies  $ b_i-\delta/2 \leq r_i^+ $. On the other hand,  we can see that $r_{i}^+$ also satisfies $r_i^+\leq b_i(0)$. Indeed, if $b_{i}-\veps \leq b_i(0)$ this is immediate. If on the other hand, $b_{i}-\veps > b_i(0)$ we see that for all $t \in [b_i(0), b_i-\veps]$ 
	 \[  \int_{t}^{b_i-\veps} w_1 \rho_1(x)dx < \int_{t}^{b_i+\veps}w_0\rho_0(x) dx, \]
	 because in the interval $( b_i(0), b_i+\veps)$ we have $w_0\rho_0 > w_1 \rho_1$. Therefore, $r_i^+ \leq b_i(0)$ in this case too.  In summary, 
	 \begin{equation}
r_i^+ \in [b_i-\delta/2, b_i(0) ].
\label{eqn:r+}
	 \end{equation}

	 Now we define the function $\phi_{b_i}: [r_i^+, b_i-\veps] \rightarrow [r_i^+ , b_i+\veps]$ as $t \mapsto \phi_{b_i}(t)$ where $\phi_{b_i}(t)$ is the largest number in $ [r_i^+, b_i-\veps]$ which satisfies:
	\[  \int_{t}^{b_i-\veps} w_1 \rho_1(x)dx = \int_{ \phi_{b_i}(t)}^{b_i+\veps}w_0\rho_0(x) dx.  \]
	Due to inequality \eqref{eqn:AuxDensities1}, $\phi_{b_i}$ satisfies:
	\begin{equation}
	|t-\phi_{b_i}(t)| \leq 2 \veps , \quad \forall t \in [r_i^+ , b_i-\veps].
	\label{eqn:PhibiIneq}
	\end{equation}
	The map $\phi_{b_i}$ induces a  measure $\gamma_{b_i}$ on $\R \times \R $ given by
	\[  \gamma_{b_i}:= (Id \times \phi_{b_i} )_{\sharp } \left(  w_1\rho_1 \measurerestr [r_i^+, b_i-\veps]  \right), \]
	whose first and second marginals are the measures $w_1\rho_1 \measurerestr [r_i^+, b_i-\veps]$ and $w_0\rho_0 \measurerestr [r_i^+, b_i+\veps]$ respectively; in the above $\sharp$ denotes the push-forward operation and $\measurerestr$ the restriction of a measure to a given set. \blue Here we recall that the push-forward of a measure $\mu$ (defined over a space $\Omega_1$) by a map $F:\Omega_1 \to \Omega_2$ is a measure on $\Omega_2$ defined by $F_\sharp \mu (B) = \mu(F^{-1}(B))$, where by $F^{-1}$ is the inverse image. \nc We also consider the inverse coupling $\gamma_{b_i}^{-1}$ defined according to the identity
	\[ \gamma_{b_i}^{-1}( D\times D') := \gamma_{b_i}( D'\times D), \]
	for all $D,D'$ measurable subsets of $\R$.

	\textbf{Step 2: \blue [Construct matching from the right of $b_i$]} In Step 2 we are repeating the same type of construction that we did in Step 1, but to the other side of the boundary point. In terms of the illustration in Figure \ref{fig:OTPlan}, we are now describing the transportation of the blue mass at the bottom of the figure. As in Step 1, we have to guarantee that such a mapping exists and transports mass at most distance $2\e$.\nc
	
	To achieve this goal, we consider a symmetric construction to the one from Step 1. Using again \eqref{eqn:densitiesDerivatives} and combining with the necessary condition for $b_i$ in \eqref{eqn:1D-Necessary} we obtain: 
	\begin{equation}
	w_1\rho_1( b_{i}-\veps + s)  \leq w_0\rho_0(b_{i}+\veps + s), \quad \forall s \in (0, \delta/2).
	\label{eqn:AuxDensities3}
	\end{equation}
	We let $\tilde r_i^+$  be the smallest number larger than $b_i+\veps$  that satisfies:
	\[ \int_{b_i-\veps}^{\tilde r_i^+} w_1\rho_1(x)dx = \int_{b_i+\veps}^{\tilde r_i^+}w_0\rho_0(x) dx.  \]
	This quantity can be shown to exist and to satisfy
	 \begin{equation}
\tilde r_i^+ \in [b_i(0), b_i+\delta/2]
\label{eqn:r+b}
\end{equation}
	using similar arguments to the ones employed in Step 1 \blue (including utilizing Assumption \ref{assump:densities}).  \nc
	
	We let $\tilde \phi_{b_i}: [ b_i-\veps, \tilde r_i^+] \rightarrow [ b_i+\veps, \tilde r_i^+ ]$ be the function defined as $t \mapsto \tilde  \phi_{b_i}(t)$ where $ \tilde \phi_{b_i}(t)$ is the smallest number in $ [b_i+\veps, \tilde r_i^+]$ which satisfies:
	\[  \int_{b_i-\veps}^{t} w_1 \rho_1(x)dx = \int_{b_i+\veps}^{ \tilde \phi_{b_i}(t)}w_0\rho_0(x) dx.  \]
	Inequality \eqref{eqn:AuxDensities3} implies 
	\begin{equation}
	|t-\tilde \phi_{b_i}(t)| \leq 2 \veps , \quad \forall t \in [b_i-\veps, \tilde r_i^+  ].
	\label{eqn:PhibiIneq2}
	\end{equation}
	The map $\tilde \phi_{b_i}$ induces a  measure $\tilde \gamma_{b_i}$ on $\R \times \R $ given by
	\[  \tilde \gamma_{b_i}:= (Id \times \tilde \phi_{b_i} )_{\sharp } \left(  w_1\rho_1 \measurerestr [b_i-\veps, \tilde r_i^+ ]  \right), \]
	whose first and second marginals are the measures $w_1\rho_1 \measurerestr [ b_i-\veps,\tilde r_i^+]$ and $w_0\rho_0 \measurerestr [ b_i+\veps, \tilde r_i^+]$ respectively. We also consider the inverse coupling $\tilde \gamma_{b_i}^{-1}$.


	\textbf{Step 3: \blue [Construct matchings for $a_i$\nc]} So far we have constructed measures $\gamma_{b_i}, \gamma^{-1}_{b_i}, \tilde{\gamma}_{b_i} , \tilde{\gamma}_{b_i}^{-1}$ relative to a finite right endpoint $b_i$, but following a completely analogous scheme, \blue and again utilizing Assumption \ref{assump:densities}, \nc we can introduce measures $\gamma_{a_i}, \gamma^{-1}_{a_i}, \tilde{\gamma}_{a_i} , \tilde{\gamma}_{a_i}^{-1}$ satisfying completely equivalent properties to their $a_i$ counterparts. In particular,  for a finite left endpoint $a_i$ we introduce two quantities $r_i^-$ and $\tilde r_{i}^-$ that satisfy
	\[  r_i^- \in  [ a_i(0), a_i+\delta/2  ], \quad   \tilde r_i^- \in  [  a_i-\delta/2, a_i(0)  ] , \]
	\begin{equation*}
	\int_{\tilde r_i^-}^{a_i+\veps} w_1 \rho_1(x)dx = \int_{\tilde r_i^-}^{a_i-\veps}w_0\rho_0(x) dx, \quad \int_{a_i+\veps}^{ r_i^-} w_1 \rho_1(x)dx = \int_{a_i-\veps}^{ r_i^-}w_0\rho_0(x) dx.
	\end{equation*}
	Two maps $\phi_{a_i}: [a_i+\veps, r_{i}^-] \rightarrow  [a_i-\veps, r_{i}^-] $  and  $\tilde \phi_{a_i}: [\tilde r_{i}^-, a_i+\veps ] \rightarrow  [ \tilde r_{i}^-,a_i-\veps] $  satisfying 
	\[  	|t-\phi_{a_i}(t)| \leq 2 \veps , \quad \forall t \in [a_i+\veps, r_{i}^-], \quad 	|t- \tilde \phi_{a_i}(t)| \leq 2 \veps , \quad \forall t \in [\tilde r_i^- , a_i+\veps].  \]
	 can be constructed. These maps induce the couplings 
	 \[ \gamma_{a_i}= (Id \times  \phi_{a_i} )_{\sharp } \left(  w_1\rho_1 \measurerestr [a_i+\veps, r_i^- ]  \right) \quad \text{ and } \quad  \tilde{\gamma}_{a_i}=(Id \times \tilde \phi_{a_i} )_{\sharp } \left(  w_1\rho_1 \measurerestr [\tilde r_i^-,a_i+\veps  ]  \right).\]

%
%

	\textbf{Step 4:\blue [Construct remainder of plan and compute cost]\nc} In addition to the constructions in Steps 1-3 we introduce $\tilde r_0^+=-\infty$ and $\tilde r_{K+1}^-=+\infty  $. Also, we set $\tilde r_1^-= r_1^-=-\infty$ in case $a_1(0)=-\infty$ and $r_K^+=\tilde r_K^+=+\infty$ in case $b_K(0)=+\infty$. We now define the desired transport plan $\pi_\veps$. 
	
	Let $\nu_0^R, \nu_1^R$ be the measures on $\R$ given by:
	\[  \nu_0^R :=    \sum_{i=1}^K  (w_1\rho_1 - w_0\rho_0)\measurerestr[r_i^-, r_i^+]  \]
	\[ \nu_1^R := \left(\sum_{i=1}^K  \left(   w_0\rho_0 - w_1 \rho_1\right) \measurerestr[\tilde r_{i-1}^+, \tilde r_{i}^-] \right)  + \left(   w_0\rho_0 - w_1 \rho_1\right) \measurerestr[\tilde r_{K}^+, \tilde r_{K+1}^-] ,\]
	we notice that since $[r_i^-, r_i^+] \subseteq [a_i(0), b_i(0)]$, $\nu_0$ is indeed a positive measure. Similarly, we can see that $\nu_1^R$ is a positive measure too . From our construction it follows that 
	\begin{equation}\int_{r_i^+}^{\tilde r_i^+} w_0 \rho_0 \,dx =\int_{r_i^+}^{\tilde r_i^+} w_1 \rho_1 \,dx,\qquad \int_{\tilde r_i^-}^{ r_i^-} w_0 \rho_0 \,dx =\int_{\tilde r_i^-}^{ r_i^-} w_1 \rho_1 \,dx,\label{eqn:balanced-transitions}\end{equation}
	for all $i$, and hence 
	\[\nu_0^R(\R) = \nu_1^R(\R)+ w_1-w_0.\] 
	Let $\pi^R$ be \textit{any} coupling between the measures
	\[  (\nu_0^R \otimes \delta_{0} + \nu_1^R \otimes \delta_{1}  ),  \text{ and } (  \nu_1^R \otimes \delta_{0}  + \nu_0^R \otimes \delta_{1}  ); \]
	notice that $\pi^R$ is a measure  on $(\R \times \{ 0,1\})^2$. This is always possible using a product coupling. In the above $\otimes$ is used to denote the product of two measures.

	Let $\pi^0$ be the measure on $(\R \times \{0,1\})^2 $ given by:
	\begin{align*}
	\pi^0(dz_1, dz_2) & := \sum_{i=1}^K (  \gamma_{b_i}(dx_1,dx_2) \otimes  \delta_{ \{0\} \times \{ 0\} }(dy_1,d y_2) + \gamma_{b_i}^{-1}(dx_1,dx_2)  \otimes  \delta_{ \{1\} \times \{ 1\} }(dy_1,d y_2)
	\\& +  \tilde \gamma_{b_i}(dx_1,dx_2) \otimes  \delta_{ \{0\} \times \{ 0\} }(dy_1,d y_2) + \tilde \gamma_{b_i}^{-1}(dx_1,dx_2)  \otimes  \delta_{ \{1\} \times \{ 1\} }(dy_1,d y_2) 
	\\&+\gamma_{a_i}(dx_1,dx_2) \otimes  \delta_{ \{0\} \times \{ 0\} }(dy_1,d y_2) + \gamma_{a_i}^{-1}(dx_1,dx_2)  \otimes  \delta_{ \{1\} \times \{ 1\} }(dy_1,d y_2)
	\\& +  \tilde \gamma_{a_i}(dx_1,dx_2) \otimes  \delta_{ \{0\} \times \{ 0\} }(dy_1,d y_2) + \tilde \gamma_{a_i}^{-1}(dx_1,dx_2)  \otimes  \delta_{ \{1\} \times \{ 1\} }(dy_1,d y_2)).
	\end{align*}
	The first and fourth terms in this expression with eight terms are the mass exchanges illustrated in Figure \ref{fig:OTPlan}. The other terms have similar interpretations. Finally, we let $\pi^F$ be the measure on  $(\R \times \{0,1\})^2 $ given by $\pi^F:= (Id \times Id)_{\sharp} ( \nu - (\pi^0+\pi^R )_1   ) $, where $(\pi^0+\pi^R )_1 $ is the first marginal of $\pi^0+\pi^R$.

	With all the above definitions in hand, we can now introduce:
	\begin{align*}
	\begin{split}
	\pi_\veps := \pi^0 + \pi^R +\pi^F .
	\end{split}
	\end{align*}
	Here, $\pi^0$ satisfies the property that for all the points in its support  $c_\veps=0$. $\pi^F$ corresponds to the mass that is fixed and thus does not contribute to the cost of $\pi_\veps$. Finally, $\pi^R$ corresponds to the remaining mass. Our construction then guarantees that 
	\[ \int c_\veps(z_1, z_2) d\pi_\veps(z_1, z_2) = \int c_{\veps}(z_1, z_2) d\pi^R(z_1, z_2) \leq \pi^R((\R \times \{0,1 \})^2) \] 
	\[ = \nu_0^R(\R) + \nu_1^R(\R)  = 2 \nu_0^R(\R) + w_0 -w_1 = 2 \sum_{i=1}^K \int_{r^-_i}^{r_i^+} (w_1 \rho_1 - w_0 \rho_0) dx  +w_0-w_1 . \] 
 In turn, 
 \begin{align*}
 \sum_{i=1}^K \int_{r^-_i}^{r_i^+} (w_1 \rho_1 - w_0 \rho_0) dx   &=  \sum_{i=1}^K \left( \int_{a_i+\veps}^{b_i-\veps} w_1 \rho_1 dx - \int_{a_i-\veps}^{b_i+\veps} w_0 \rho_0 dx    \right) 
 \\& =  \int_{(A^*)^{-\veps}} w_1 \rho_1 dx - \int_{(A^*)^{\veps}} w_0 \rho_0 dx ,
 \end{align*}
thanks to equation \eqref{eqn:BalancingCondition} and the analogues for the $a_i$. Thus,
	\[  \frac{1}{2}\int c_\veps(z_1, z_2) d\pi_\veps(z_1, z_2) + \frac{1}{2}(w_1-w_0) \leq  \int_{(A^*)^{-\veps}} w_1 \rho_1 dx - \int_{(A^*)^{\veps}} w_0 \rho_0 dx,  \]
	which thanks to Remark \ref{rem:Duality} implies that $A^{*}$ solves the optimization problem and that the above inequality is actually an equality.
	
\end{proof}

\nc

\begin{remark}
  The construction of transportation plans in the previous proof is possible due to the necessary conditions \eqref{eqn:1D-Necessary} that are maintained by the evolution equations \eqref{eqn:evolutiod-1db} and \eqref{eqn:evolutiod-1d}. Indeed, the necessary conditions are crucially used to prove the equalities in \eqref{eqn:balanced-transitions}, which factored prominently in the construction of the certifying transportation plan $\pi_\e$. 
\end{remark}

\begin{remark}
  The construction in the previous proof is local, in the sense that we can only show that solutions to our evolution equations are global minimizers for $\e$ sufficiently small. However, the proof of the previous proposition indicates some situations where one can detect that these solutions cease to be global minimizers. For example, if at some point $r_i^+ = \tilde r_{i+1}^-$ then one expects that the construction may not be continued for larger $\e$. This should correspond to a change in topology of the global optimizer. \blue On the other hand, the fact that global minimizers are characterized by the differential equation implies that the topology of the minimizers does not change for small values of $\e$, and if one chose to track the range of values for which the constructions in the previous proof were valid (which would depend upon the difference between $w_0\rho_0'$ and $w_1 \rho_1'$ and upon the $C^1$ norms of the densities) one could quantify the range of $\e$ for which the topology does not change.\nc Understanding the type of degeneracies that may arise when solving the geometric evolution equations, \blue and the associated changes in topology of the optimizers,\nc as well as their implications to the adversarial risk minimization problem are topics of current investigation. 
\end{remark}

\nc

%

\nc

%
%
%
%

\section{Necessary conditions and geometric evolution equations in higher dimension}\label{sec:high-D}

In one dimension, the necessary condition allowed us to derive an ordinary differential equation that described the motion of decision boundaries as we increased the adversarial power $\e$. This evolution equation was driven, for small $\e$, by the gradient of $\rho$. In higher dimension the optimality conditions and their associated geometric evolution equations are necessarily more complex. In particular, the presence of curvature in higher dimensions introduces a greater degree of complexity. \blue For clarity, throughout our study of dimension $d>1$ we will restrict our attention to the standard Euclidean metric, namely we let $d(x_1,x_2) = |x_1-x_2|$. \nc

To begin, we will develop some intuition about the problem by studying an explicit, radial example.

\begin{example}
Let us consider the case where $\rho$ is a uniform distribution on a ball of radius $1$ in $\R^d$, and $w_0\rho_0(x) = \frac{|x|}{\omega_d}$, with $\omega_d$ the $\mathcal{L}^d$ measure of the unit ball. Here the Bayes classifier is given by $u_B(x) = \mathds{1}_{|x|\leq 1/2}$. We then consider a classifier, parameterized in $\e$, which (by way of ansatz) is given by $\mathds{1}_{|x| \leq r(\e)}$, which minimizes the adversarial cost. \blue Necessary conditions for optimality then take the form
  \begin{align*}
  0 &= \lim_{\delta \to 0} \frac{ R_\e(r(\e)+\delta) - R_\e(r(\e))}{\delta}\\
 &= \omega_d  w_0 \rho_0(r(\e)+\e) \cdot (r(\e)+\e)^{d-1} - \omega_d  w_1 \rho_1(r(\e) - \e) \cdot ( r(\e)-\e)^{d-1} \\
        &=(r(\e) + \e) \cdot (r(\e) + \e)^{d-1} - (1-(r(\e) - \e)) \cdot ( r(\e)-\e)^{d-1}, \\
      \end{align*}
      where here we are abusing notation slightly and writing $\rho_1(t)$ and $\rho_0(t)$ to represent $\rho_1(x)$ and $\rho_0(x)$ for all $x$ such that $|x|=t$, and writing $R_\e(s) = R_\e(\mathds{1}_{|x| \leq s})$.\nc Taking a derivative in $\e$ we obtain
\[
  d (r(\e) + \e)^{d-1} \left( \frac{d}{d\e} r + 1\right) =   \left((d-1)(r(\e) - \e)^{d-2} - d(r(\e) - \e)^{d-1} \right) \left( \frac{d}{d\e} r - 1\right),
\]
which may be written
\[
\frac{dr}{d\e} = -\frac{d (r(\e) + \e)^{d-1} + \left((d-1)(r(\e) - \e)^{d-2} - d(r(\e) - \e)^{d-1} \right)}{d (r(\e) + \e)^{d-1} - \left((d-1)(r(\e) - \e)^{d-2} - d(r(\e) - \e)^{d-1} \right)}.
\]
At $\e=0$ this becomes
\[
\frac{dr}{d\e} (\e = 0) = -\frac{(d-1)r^{d-2}}{2d r^{d-1} - (d-1)r^{d-2}} = - \frac{(d-1)r^{-1}}{2d - (d-1)r^{-1}}
\]
Recalling that $r^{-1} = \kappa$ is the mean curvature of a sphere of radius $r$ in $\R^d$, we immediately see the effect of curvature, namely that this evolution corresponds, at $\e=0$ to a mean curvature flow \blue that has been reweighted in the denominator by a density-dependent factor\nc. We notice that here, $\nabla \rho \equiv 0$, which in the one dimensional case dominated the evolution for small $\e$ regimes. This example was specifically chosen in order to highlight the effect of curvature, but we will subsequently see that both curvature and $\nabla \rho$ play a role in the surface evolution.

\blue We remark that, in order to make the formulas explicit, we choose to work with the uniform density on the ball, which has non-smooth density. We notice that the classification boundary occurs in the region where the density is smooth, and so mollifying the density near the boundary of the ball would not materially affect the behavior of the example besides complicating the formulas for the total density. \nc


\end{example}

With the previous example in mind, we derive the necessary condition, assuming that the decision boundary is sufficiently smooth.

\begin{proposition}\label{prop:nec-nD}
  Suppose that $A_\e$ is a critical point of the problem $R_\e$ (with respect to normal variations \cite{Maggi-Book}, further description given in the proof below) and that the signed distance function $\tilde d_{A_\e}$ is $C^3$ on the set $|\tilde d_{A_\e}| < 2\e$. For $x \in \partial A_\e$, let $\nu$ denote the outward unit normal and $\kappa_i$ denote the principal curvatures \blue (see the Appendix for a definition)\nc. Then the following necessary condition holds for almost every $x \in \partial A_\e$:
\begin{equation}\label{eqn:nD-necessary}
w_1 \rho_1(x -\e \nu(x)) \prod_{i=1}^{d-1} |1 - \kappa_i \e| - w_0\rho_0(x +\e \nu(x)) \prod_{i=1}^{d-1} |1 +  \kappa_i \e| = 0.
\end{equation}

\end{proposition}
\nc

\begin{proof}
We again recall
\begin{align*}
  R_\e(\mathds{1}_A) &= \int_{\tilde d_A(x)<-\e} w_0(x)\rho_0(x) \,dx + \int_{\tilde d_A(x) > \e} w_1\rho_1(x)\,dx \\
  &+ \int_{|\tilde d_A(x)| < \e} \rho(x)\,dx.
\end{align*}
We consider the class of normal variations \cite{Maggi-Book} of the set $A=A_\veps$: that is, we consider a one parameter family of sets $A^t$ of the form $A^t = \phi(t,A)$ for some diffeomorphism $\phi(t,x)$ which satisfies $\phi(0,A)=A$ and $\frac{d\phi}{dt}(t=0) = F(x)$, where $F$ satisfies $F(x) = \nu(x) \psi(x)$ for $x \in \partial A$ and for some smooth scalar valued function $\psi$ \blue that we assume, without loss of generality, satisfies $\nabla \psi(x) \cdot \nu(x)=0 $ for all $x \in \partial A$ \nc. Taking the derivative of $R_\e (\mathds{1}_{A^t})$ and evaluating it at $t = 0$, we obtain that
\[
  0 =   \int_{\tilde d_A(y) = \e} w_0 \rho_0(y) \psi(P_{\partial A}(y)) \,d\mathcal{H}^{d-1}(y)- \int_{\tilde d_A(y) = -\e} w_1 \rho_1(y) \psi(P_{\partial A}(y)) \,d\mathcal{H}^{d-1}(y),
\]
where here $P_{\partial A}(x)$ is the projection of $x$ onto the boundary of $A$, \blue meaning the point in the boundary of $A$ which is closest to $x$, whose uniqueness is guaranteed by the assumption upon the regularity of the signed distance function.\nc

Noting that $y = x \pm \e \nu(x)$ in the previous two integrals, we then use a change of variables as in Corollary \ref{cor:change-of-variables} to convert to
\[
  0 = \int_{\tilde d_A(x) = 0} \left(w_0\rho_0(x + \e \nu(x)) \prod_{i=1}^{d-1} |1 + \kappa_i(x) \e| - w_1\rho_1(x - \e \nu(x))\prod_{i=1}^{d-1} |1 - \kappa_i(x) \e| \right) \psi(x) \,d\mathcal{H}^{d-1}(x).
\]
Since this holds for all smooth $\psi$, we then have that, for $\mathcal{H}^{d-1}$ almost every $x \in \partial A$
\begin{equation}\label{eqn:multi-d-necessary}
0 = \left(w_0\rho_0(x + \e \nu(x)) \prod_{i=1}^{d-1} |1 + \kappa_i(x) \e| - w_1\rho_1(x - \e \nu(x)) \prod_{i=1}^{d-1} |1 - \kappa_i(x) \e| \right)
\end{equation}
\end{proof}

We \blue notice \nc that assuming that the conditional densities $\rho_0,\rho_1$ are smooth is not sufficient to guarantee that the set of $x$'s for which $w_0\rho_0 -w_1\rho_1 = 0$ is smooth, as evidenced by the following basic example:

\begin{example}
  Suppose in $\R^2$ that one places normals associated with $y=+1$ at $(1,1)$ and $(-1,-1)$, and then places normals associated with $y=-1$ at $(1,-1)$ and $(-1,1)$, with $w_0 = w_1 = 1/2$. In this case the set where $w_0\rho_0 =w_1\rho_1$ is given by the set $\{x=0\} \cup \{y=0\}$, which is not smooth at $(0,0)$.
\end{example}

\blue
  In Proposition \ref{prop:nec-nD} we notice that the necessary condition directly utilizes the normal vectors to the surface and the curvatures (which may be viewed as derivatives of the normal vectors). These notions are intimately tied with the classical Euclidean geometry: indeed if we did not have $d(x_1,x_2) = |x_1=x_2|$ then the previous theorem would need to be modified in order to accommodate for normal vectors and their derivatives in the appropriate geometry. Such definitions have been pursued in the context of mean curvature flow, for example in \cite{bellettini2004anisotropic}. However, extending those definitions to apply to the present context is beyond the scope of this work.
\nc

\subsection{Geometric flow}\label{sec:multi-d-evolution}

In this section we seek to formally derive a geometric flow which characterizes the evolution of the boundary of the $A_\e$. As in the one-dimensional case, we can Taylor expand for $\e$ small to derive an approximating geometric flow which is more transparent and easier to interpret.

To begin, let us suppose that $\phi(\e,x)$ be a diffeomorphism so that $\phi(\e,A) = A_\e$. We shall utilize the necessary condition \eqref{eqn:nD-necessary} to characterize this diffeomorphism for points $x \in \partial A_0$.

We now use a chain rule on the necessary condition as follows (suppressing the dependence on $x,\e$, and always assuming that $x \in \partial A_0$):
\begin{align}\label{eqn:d-dim-evolution-exact}
\begin{split}
  &0 = \frac{d}{d\e}\left( w_0\rho_0(\phi + \e \nu(\phi)) \prod_{i=1}^{d-1} (1 + \e \kappa_i(\phi)) - w_1\rho_1(\phi - \e \nu(\phi) \prod_{i=1}^{d-1} (1 - \e \kappa_i(\phi)) \right) \\
    &= \prod_{i=1}^{d-1} (1 + \e \kappa_i(\phi)) \left( \nabla w_0\rho_0(\phi + \e \nu(\phi) \left(\frac{d}{d\e} \phi + \nu(\phi) + \e \frac{d}{d\e} (\nu(\phi))\right) + w_0\rho_0(\phi + \e \nu(\phi))\sum_i \frac{\kappa_i(\phi) + \e \frac{d}{d\e} \kappa(\phi)}{1+\e \kappa_i(\phi)} \right) \\
      &-\prod_{i=1}^{d-1} (1 - \e \kappa_i(\phi)) \left( \nabla w_1\rho_1(\phi - \e \nu(\phi) \left(\frac{d}{d\e} \phi - \nu(\phi) - \e \frac{d}{d\e} (\nu(\phi))\right) + w_1\rho_1(\phi - \e \nu(\phi))\sum_i \frac{-\kappa_i(\phi) - \e \frac{d}{d\e} \kappa(\phi)}{1-\e \kappa_i(\phi)} \right)
      \end{split}
\end{align}
One major challenge here is that $\frac{d}{d\e} \nu(\phi)$ and $\frac{d}{d\e} \kappa$ will involve mixed derivatives, i.e. derivatives in both $\e$ and $x$. Indeed, we recall \blue (e.g. Section 17.1 in \cite{Maggi-Book})  that, in terms of $\phi$ and for $x \in \partial A_0$, one may express the geometric quantity $\nu$ (the outward surface normal) as
\begin{align*}
   \nu(\phi(\e,x)) &= \frac{(\frac{\partial}{\partial x} \phi(\e,x))^{-T} \cdot \nu_0(x)}{ |(\frac{\partial}{\partial x} \phi(\e,x))^{-T} \cdot \nu_0(x)| }\nc
\end{align*}
\blue Similarly, the curvatures $\kappa_i$ may be expressed in terms of appropriate spatial derivatives of $\nu$, \blue more precisely, as the non-trivial eigenvalues of the matrix $\frac{\partial \nu}{\partial x}$; see Proposition \ref{eq:PropsDistance} in the Appendix. \nc Thus for $\e>0$ this evolution equation is a non-local, mixed-type partial differential equation, which appears difficult to solve.

However, each of the terms involving mixed derivatives is pre-multiplied by $\veps$, and hence may plausibly be ignored for $\e$ sufficiently small. To this end, we rearrange the previous equation
\begin{align}\label{eqn:d-dim-evolution}
  \begin{split}
  &\left( \prod_{i=1}^{d-1} (1 + \e \kappa_i(\phi)) \nabla w_1\rho_1(\phi + \e \nu(\phi)) - \prod_{i=1}^{d-1} (1 - \e \kappa_i(\phi)) \nabla w_0\rho_0(\phi - \e \nu(\phi))\right) \frac{d}{d\e}\phi \\
  &= - \prod_{i=1}^{d-1} (1 + \e \kappa_i(\phi)) \left( \nabla w_1\rho_1(\phi + \e \nu(\phi) ( \nu(\phi) + \e \frac{d}{d\e} \nu(\phi)) + w_1\rho_1(\phi + \e \nu(\phi))\sum_i \frac{\kappa_i(\phi) + \e \frac{d}{d\e} \kappa(\phi)}{1+\e \kappa_i(\phi)} \right) \\
&-\prod_{i=1}^{d-1} (1 - \e \kappa_i(\phi)) \left( \nabla w_0\rho_0(\phi - \e \nu(\phi) (  \nu(\phi) + \e \frac{d}{d\e} \nu(\phi)) + w_0\rho_0(\phi - \e \nu(\phi))\sum_i \frac{\kappa_i(\phi)  +\e \frac{d}{d\e} \kappa(\phi)}{1-\e \kappa_i(\phi)} \right).
\end{split}
\end{align}
Evaluating at $\e=0$, we find that
\[
  (w_1 \nabla \rho_1 - w_0 \nabla \rho_0) \frac{d\phi}{d\e} = - \left(\nabla \rho\cdot \nu + \rho \sum_i \kappa_i\right).
\]
If we express $\frac{d \phi}{d\e} = v \nu$, namely we consider the normal speed $v$, then we may write
\begin{equation}\label{eqn:d-dim-evolution-approx}
v(x,\e=0) =  -\frac{\nabla \rho\cdot \nu + \rho \sum_i \kappa_i}{(w_1 \nabla \rho_1 - w_0 \nabla \rho_0)\cdot \nu}
\end{equation}
Here we observe two terms: one which induces motion ``downhill'' in $\rho$ and a second which is a positively weighted mean curvature term. As we have used $\nu$ as an outwardly pointing normal vector, the $-\sum \kappa$ will correspond to the standard mean curvature flow. This indicates that heuristically, near $\e=0$, the optimal adversarial classifier seeks to i) go downhill in $\rho$, and ii) decrease the perimeter of the decision boundary (since mean curvature flow is a type of gradient flow of perimeter; see Section \ref{sec:Perimeter} below.) While the reweighing in the denominator is not homogeneous, and indeed makes this heuristic description imprecise, we believe this heuristic picture is helpful for understanding the local effects induced by adversarial robustness.

\nc

\blue Mean curvature flow, namely the case where $v(x) = \sum_i \kappa_i$ without any density weighting, is a fundamental geometric flow. One reference text, among many, on the topic is \cite{mantegazza2011lecture}. This geometric flow is known to be the gradient flow of the perimeter or area functional with respect to a certain function space. It is also known to induce increased smoothness in surfaces, when measured in the correct function spaces. Finally, mean curvature flow obeys a comparison principle and admits efficient numerical methods. While the flow that we have derived for the adversarial problem does not match mean curvature flow exactly, one of the terms in the approximate surface evolution at $\e=0$ amounts to a scalar function times mean curvature flow, suggesting that the flow induced by the adversarial problems also may enjoy similar useful characteristics. In the next section we take this analogy one step further by deriving a variant of perimeter regularization, which matches the adversarial evolution equation to higher order.
\nc

\blue

\subsubsection{Connection with explicit perimeter regularization}
\label{sec:Perimeter}

As mentioned at the end of the Introduction, there is a close relationship between the evolution equation derived earlier and the one that one would obtain by tracking solutions to the family of problems \eqref{eqn:PerimeterMinimization} indexed with $\veps$. In what follows we elaborate on this statement. It will be convenient to write $R$ explicitly as:
\[ R(\mathds{1}_A)= \int_{A} w_0 \rho_0(x) dx + w_1 - \int_{A} w_1 \rho_1(x)dx. \]

First, let us derive necessary conditions for an optimal solution $A$ to problem $\eqref{eqn:PerimeterMinimization}$. We assume that $A$ has at least $C^2$ boundary for simplicity. Let $\phi$ be a normal variation of $A$ as considered in Section \ref{sec:multi-d-evolution} with the same notation used there. Then the following two relations hold:
\begin{align*}
   \frac{d}{dt} \Bigr|_{t=0} R(\mathds{1}_{A^t}) & =  \frac{d}{dt} \Bigr|_{t=0} \left( \int_{A^t} (w_0 \rho_0(x) - w_1 \rho_1(x) ) dx   \right)  
   \\&= \int_{\partial A} \psi(x) (w_0  \rho_0 (x) - w_1  \rho_1 (x) ) d \HH^{d-1}(x),
\end{align*}
and
\begin{align*}
   \frac{d}{dt} \Bigr|_{t=0} \Per_\rho(A^t) & =  \frac{d}{dt} \Bigr|_{t=0} \left( \int_{\partial A^t} \rho( z ) d\HH^{d-1}(z)   \right)  
   \\&=  \frac{d}{dt} \Bigr|_{t=0} \int_{\partial A} \rho \left(x+ t \psi(x) \nu(x)\right)    \left| \det \left( I + t \psi(x) \frac{\partial \nu(x)}{\partial x} \right) \right| \HH^{d-1}(x)
   \\&=\int_{\partial A} \psi(x) ( \nabla \rho(x) \cdot \nu(x)  +   \rho(x) \sum_{i}\kappa_i ) d\HH^{d-1}(x). 
\end{align*}
We conclude that
\begin{align*}
0 & =  \frac{d}{dt} \Bigr|_{t=0} \left( R(\mathds{1}_{A^t}) + \veps  \Per_\rho(A^t)  \right)
\\& =  \int_{\partial A} \psi(x) ( w_0  \rho_0 (x) - w_1 \rho_1 (x)  + \veps (\nabla \rho(x) \cdot \nu(x)  +   \rho(x) \sum_{i}\kappa_i))   d \HH^{d-1}(x). 
\end{align*}
Since the normal variation was arbitrary, and thus the scalar function $\psi$ was too, we deduce the necessary condition:
\begin{equation}
0 =  w_0  \rho_0 (x) - w_1 \rho_1 (x) + \veps (\nabla \rho(x) \cdot \nu(x)  +   \rho(x) \sum_{i}\kappa_i), \quad x \in \partial A.
\label{eq:NecessaryRegularization}
\end{equation}

Let $A_\veps$ be a solution to problem \eqref{eqn:PerimeterMinimization} for a given value of $\veps$ and let us assume that the family $\{ A_\veps \}_{\veps>0}$ can be represented via a one-parameter family of diffeomorphisms $\{ \phi(\veps, \cdot) \}_{\veps>0}$ so that $\phi(\veps, A) =  A_\veps$, where $A$ is the set induced by the Bayes classifier. As in Section \ref{sec:multi-d-evolution}, we now use the necessary conditions \eqref{eq:NecessaryRegularization} at each point $x$ to characterize the family of diffeomorphisms at $\veps=0$ and $x \in \partial A$. Indeed, differentiating the equation
\begin{align*}
0 & =  w_0 \rho_0 (\phi(\veps,x)) - w_1  \rho_1 ( \phi(\veps, x)) 
\\& + \veps (\nabla \rho(\phi(\veps, x)) \cdot \nu(\phi(\veps, x))  +   \rho(\phi(\veps, x)) \sum_{i}\kappa_i(\phi(\veps, x))), \quad x \in \partial A, \quad \veps \geq 0
\end{align*}
with respect to $\veps$, and setting $\veps=0$, we deduce:
\[ 0 = (w_0 \nabla \rho_0(x) - w_1 \nabla \rho_1(x)) \cdot \frac{d\phi}{d\veps}(0, x) + (\nabla \rho(x) \cdot \nu(x)  +   \rho(x) \sum_{i}\kappa_i(x)), \quad x \in \partial A.  \]
Expressing $\frac{d\phi}{ d \veps} = v \nu(x)$, we obtain \eqref{eqn:d-dim-evolution-approx}, i.e. the same infinitesimal change as the one coming from the original adversarial problem.

\nc

\blue
\subsection{Global minimizers in higher dimension}

\begin{theorem}
\label{thm:Multid}
    Suppose that the evolution equation \eqref{eqn:d-dim-evolution-exact} admits a classical solution for $\e \in [0,\e_0]$, in the sense that there exists a $C^2$ diffeomorphism $\phi$ which satisfies the equation at every point $x \in \partial  A_0$. Suppose, furthermore, that $\rho_0,\rho_1$ are $C^2$ (i.e. bounded first and second derivatives) and that along the entire interface of the Bayes' classifier we have
    \begin{equation}\label{eqn:strict-crossing-multi-d}
    (w_0 \nabla \rho_0 -w_1 \nabla \rho_1)\cdot \nu \geq c_0 > 0,
    \end{equation}
     \blue
    for some constant $c_0$, where $\nu=\nu(x)$ is the outer unit normal at $x \in \partial A_0$. Then for $\e$ sufficiently small the solution to the evolution equation is also a global minimizer of the adversarial problem, where the metric determining the actions of the adversary is the Euclidean metric.
\end{theorem}

\begin{proof}
To begin, we notice that, by \eqref{eqn:strict-crossing-multi-d} along with the implicit function theorem, the boundary of $A_0$ is locally the graph of a $C^2$ function. Classical geometric results imply that for any point $\tilde x$ in a $\delta$ neighborhood of $\partial A_0$ there exists a unique closest point $P_{\partial A_0}(\tilde x)$ in the boundary of $\partial A_0$, and \blue that the signed distance function is $C^3$ in that same $\delta$ neighborhood of $\partial A_0$. \blue

Under the assumption that $\phi$ is $C^2$, the chain rule computation shown in Section \ref{sec:multi-d-evolution}, along with the fact that $A_0$ is the Bayes' classifier, implies that the necessary condition \eqref{eqn:multi-d-necessary} is satisfied at every point in the boundary of $A_\e$, namely for every $x \in \partial A_\e$
\[
0 = \left(w_0\rho_0(x + \e \nu(x)) \prod_{i=1}^{d-1} |1 + \kappa_i(x) \e| - w_1\rho_1(x - \e \nu(x)) \prod_{i=1}^{d-1} |1 - \kappa_i(x) \e| \right).
\]

Next we notice that we may express points within the $\delta$ neighborhood of $\partial A_0$, using what is called a \emph{normal coordinate system}. In particular, for any $\tilde x$ in that neighborhood, we may (uniquely) represent $\tilde x = P_{\partial A_0}(\tilde x) + \tilde{d}_{A_0}(\tilde x) \nu( P(\tilde x))$. \blue This essentially allows us to locally transform from points in a neighborhood of the boundary into flattened geometry. We note that, since $\phi$ is $C^2$ and starts as the identity mapping, the set $A_\e$ also has boundary which is the graph of a $C^2$ function and admits local normal boundary coordinates, for $\e$ sufficiently small, and for a $\delta$ neighborhood of the boundary of $A_\e$ which is independent of $\e$. From this point on we will always assume that $\e$ is small enough that this representation is possible.


Our goal now will be to construct a transportation plan which certifies the optimality of the set $A_\e$. In doing so, as in the one-dimensional proof, it suffices to construct mappings locally near the boundary which transfer mass from $\rho_0$ to $\rho_1$ (and vice versa), and which move mass at most $2\e$ distance. As we will see, our construction reduces the problem almost entirely to the one-dimensional setting.\blue

\blue
In particular, fix $x_0 \in \partial A_0$ let $z_0= \phi(\veps, x_0)$, and consider points of the form $z_0 + t \nu(z_0) = z(t,z_0)$, for $t \in (-\delta,\delta)$. After changing variables (to be precise, in the boundary normal coordinates associated with $\partial A_\e$), the density associated with a particular $t$ for fixed $z_0$ is given by
\[
\tilde \rho_i(t|z_0) := \rho_i(z(t,z_0)) \prod_{i=1}^{d-1} |1+t \kappa_i(z_0) |.
\]
More precisely, using the \blue coarea formula \cite{evans2015measure}  and Corollary \ref{cor:change-of-variables} in the Appendix, we can write
\begin{align*}
\int_{ |\tilde{d}_{A_\veps}(z)| \leq \delta    }   g(z) \rho_{i}(z) dz  & = \int_{-\delta}^{\delta} \left( \int_{\tilde{d}_{A_\veps}(z) =s} g(z)\rho_i(z) d \HH^{d-1}(z) \right)  dt
\\& =\int_{-\delta}^{\delta} \left( \int_{\partial A_\veps } g(z(t,z_0)) \rho_i(z(t,z_0)) \prod_{i=1}^{d-1} |1+t \kappa_i(z_0) | d \HH^{d-1}(z) \right)  dt
\\&=  \int_{\partial A_\veps } \int_{-\delta}^{\delta} g(z(t,z_0))\tilde \rho_i(t|z_0)   dt d\HH^{d-1}(z),
\end{align*}
for arbitrary smooth and bounded test function $g$. This provides a representation of the distribution $\rho_i dx $ restricted to the set $\{ |\tilde{d}_{A_\veps}(z)| \leq \delta  \}$ in normal coordinates, and in particular, up to rescaling, we can now interpret the function $\tilde {\rho}_i(\cdot|z_0)$ as the conditional distribution of $t \in [-\delta, \delta]$ given $z_0$.

Now for any fixed $\e$ and $x_0$ (i.e. fixed $z_0= \phi(\veps, x_0)$) we will construct a transportation plan using $\tilde \rho(\cdot|z_0)$: in the original high-dimensional problem this means that in the set $\{ |\tilde{d}_{A_\veps}(z)| \leq \delta  \}$, our transportation plan only transports along rays normal to the boundary of $\partial A_\e$. Notice that outside of $\{ |\tilde{d}_{A_\veps}(z)| \leq \delta  \}$, on the other hand, we may transport in any way we want so as to match the marginal constraints (just as in the 1d setting): this is why we focus on the set $\{ |\tilde{d}_{A_\veps}(z)| \leq \delta  \}$ exclusively. 

By the necessary condition \eqref{eqn:multi-d-necessary} we have the necessary matching condition, namely that $w_0\tilde \rho_0(-\e| z_0) = w_1\tilde \rho_1(\e|z_0)$. All that remains is to verify that we can transport $w_0 \tilde \rho_0(\cdot|z_0)$ on the interval $[-\e,\e]$ on to $w_1\tilde \rho_1(\cdot |z_0)$ without moving more than distance $2\e$ in $t$. To this end, we notice that, by the assumption \eqref{eqn:strict-crossing-multi-d},
 $|(w_1\nabla \rho_1(\tilde x) -w_0\nabla \rho_0(\tilde x' )) \cdot \nu(x_0)| > c_0/2 >0$ for all $\tilde x, \tilde x'$ in a neighborhood of $x_0 \in \partial A_0$. Using the fact that $\phi$ is $C^2$ and that is the identity for $\e=0$, then implies that the same inequality, with constant $c_0/4$ holds in a neighborhood of $\phi(x_0,\e)$, for small enough $\e$, and for the size of the neighborhood independent of $\e$. We may then compute
\begin{align*}
w_1 \tilde \rho_1'(t_1|z_0) - w_0 \tilde \rho_0'(t_0|z_0) &= (w_1 \nabla \rho_1(z(t_1,z_0)) -w_0 \nabla \rho_0(z(t_0,z_0)))\cdot \nu(z_0)  \\
&+ w_1 \rho_1(z(t_1,z_0)) \sum_i \kappa_i(z_0) \prod_{j \neq i} |1+t_1\kappa_j(z_0) | 
\\& - w_0 \rho_0(z(t_0,z_0)) \sum_i \kappa_i(z_0) \prod_{j \neq i} |1+ t_0 \kappa_j(z_0) |.
\end{align*}
The first of these terms is bounded from below by $c_0/4$. Recalling that $w_0 \rho_0(x_0) = w_1\rho_1(x_0)$ (as $A_0$ is a Bayes' classifier), we may bound the magnitude of the remaining terms by $|t_i|(1+\delta)$: in particular, if $t_i$ is of order $\e$ ( notice that $\kappa_i$ is uniformly bounded since $\phi$ was assumed $C^2$) then we may conclude that  $w_1 \tilde \rho_1'(t_1|z_0) - w_0 \tilde \rho_0'(t_0|z_0) > c_0/8 $ for small enough $\e$. This then allows us to directly use the one-dimensional construction from the proof of Theorem \ref{thm:GlobalOptimality} in order to construct appropriate transportation plans for the $\tilde \rho(\cdot| z_0)$. By constructing such a plan along normal rays corresponding to every $z_0 \in \partial A_\veps$, we then have a candidate transportation plan, which transports points near the boundary at most distance $2\e$. Using the same argument via the fundamental theorem of calculus and the duality principle as in the proof of Theorem \ref{thm:GlobalOptimality}, we obtain the desired result.
\end{proof}
\nc

\blue

\begin{remark}
  In the previous proof we assumed that the solutions to the partial differential equation existed. We suspect that the local existence and uniqueness of solutions can be proved by an appropriate non-linear PDE argument under the assumption \eqref{eqn:strict-crossing-multi-d}, but the technical details of such a proof lie outside of the scope of this paper.
  
  We also notice that the same conclusions about topology which we indicated in one dimension would also hold in the setting of Theorem \ref{thm:Multid}. Indeed, minimizers will not change their topology for sufficiently small $\e$, the size of which should depend upon the size of $c_0$ in \eqref{eqn:strict-crossing-multi-d} and upon the smoothness of the underlying densities. Of course the topologies are potentially much more complicated in higher dimension, and the evolution of the topology of the minimizing classifiers is an intriguing potential future direction.
\end{remark}
\nc

\subsection{Illustration in two dimensions}\label{sec:2Dex}

Here we show a basic numerical example of the geometric evolution \eqref{eqn:d-dim-evolution-approx} in two dimensions. This example is intended to be an illustration, rather than a detailed computational study. Such a study would require careful numerical analysis, which lies outside of the scope of this work.

We consider two different classes $\rho_1 \sim N((-.5,-2),\Sigma) + N((-.5,.5),\Sigma)$ and $\rho_2 \sim N((.5,-.5),\Sigma) + N((.5,2),\Sigma$, where $\Sigma = .2 I$, and $w_0 = w_1 = .5$. The Bayes classifier boundary, along with contours of the misclassification error $w_1\rho_1 - w_0 \rho_0$ are shown in Figure \ref{fig:risk-contours}.

\begin{figure}
\centering
\includegraphics[width=.8\textwidth]{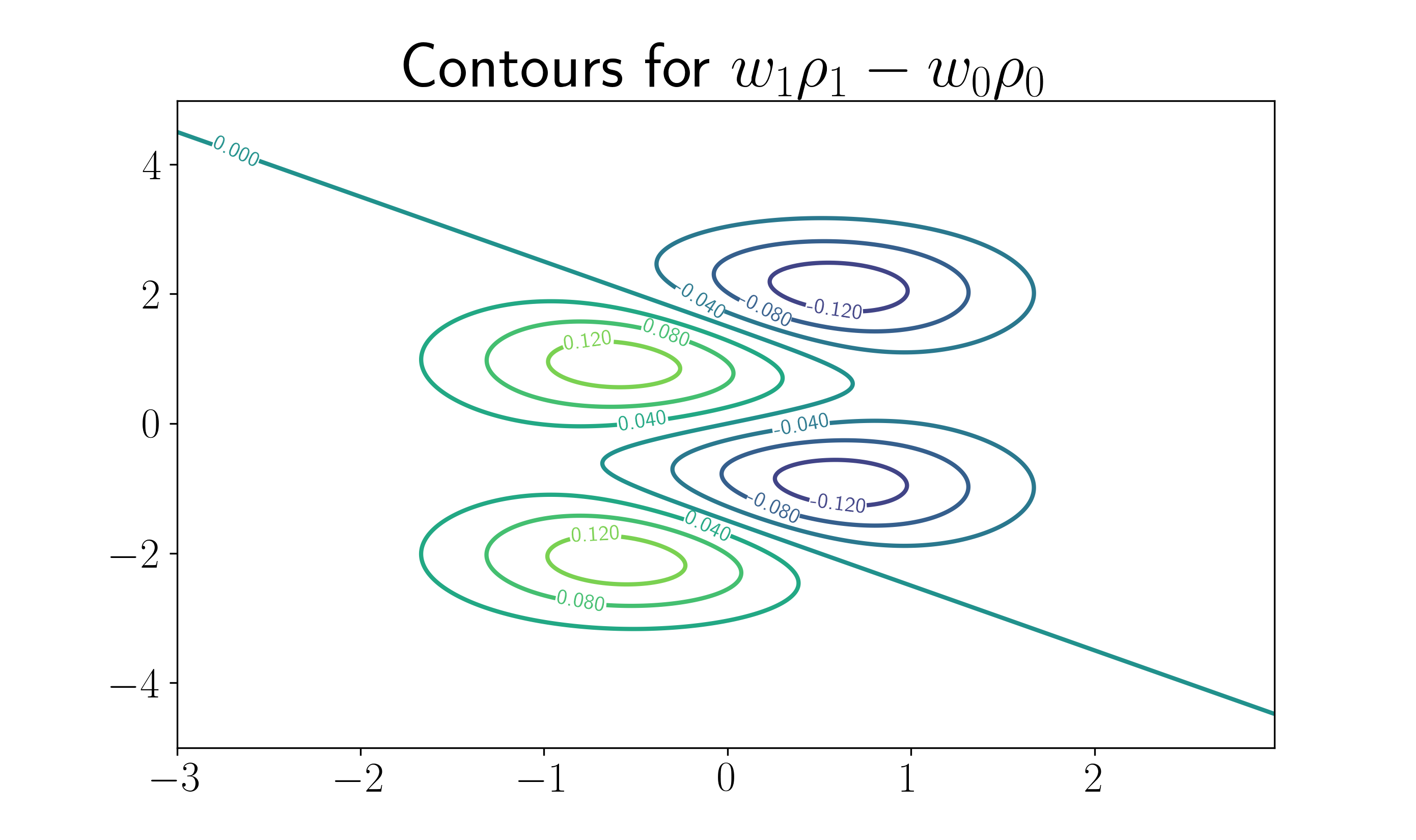}
\caption{The contours of the function $w_1\rho_1 - w_0\rho_0$ for the example in Section \ref{sec:2Dex}. The contour corresponding to $w_1\rho_1 - w_0\rho_0=0$ is the s-shaped curve, and represents the decision boundary for the Bayes classifier.}
\label{fig:risk-contours}
\end{figure}

We then use a modified version of the scheme from \cite{merriman1992diffusion} to track the evolution of the decision boundary under the evolution equation \eqref{eqn:d-dim-evolution-approx} for different values of $\e$. 
These curves are displayed in Figure \ref{fig:geometric-evolution}, next to the curves evolved via standard mean curvature flow as a point of reference.

\begin{figure}[p] 
\centering
\includegraphics[width=.8\textwidth]{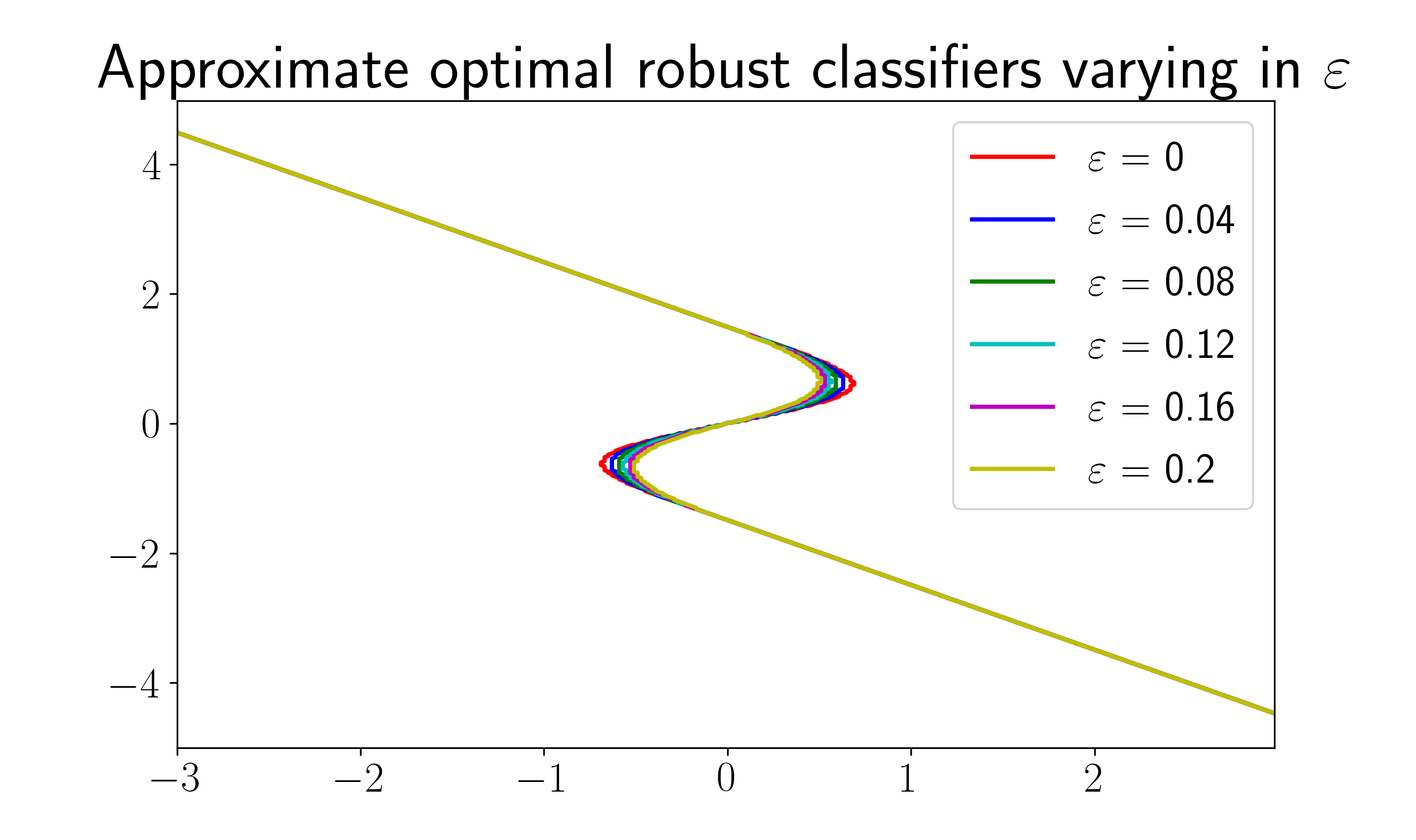}

\includegraphics[width=.8\textwidth]{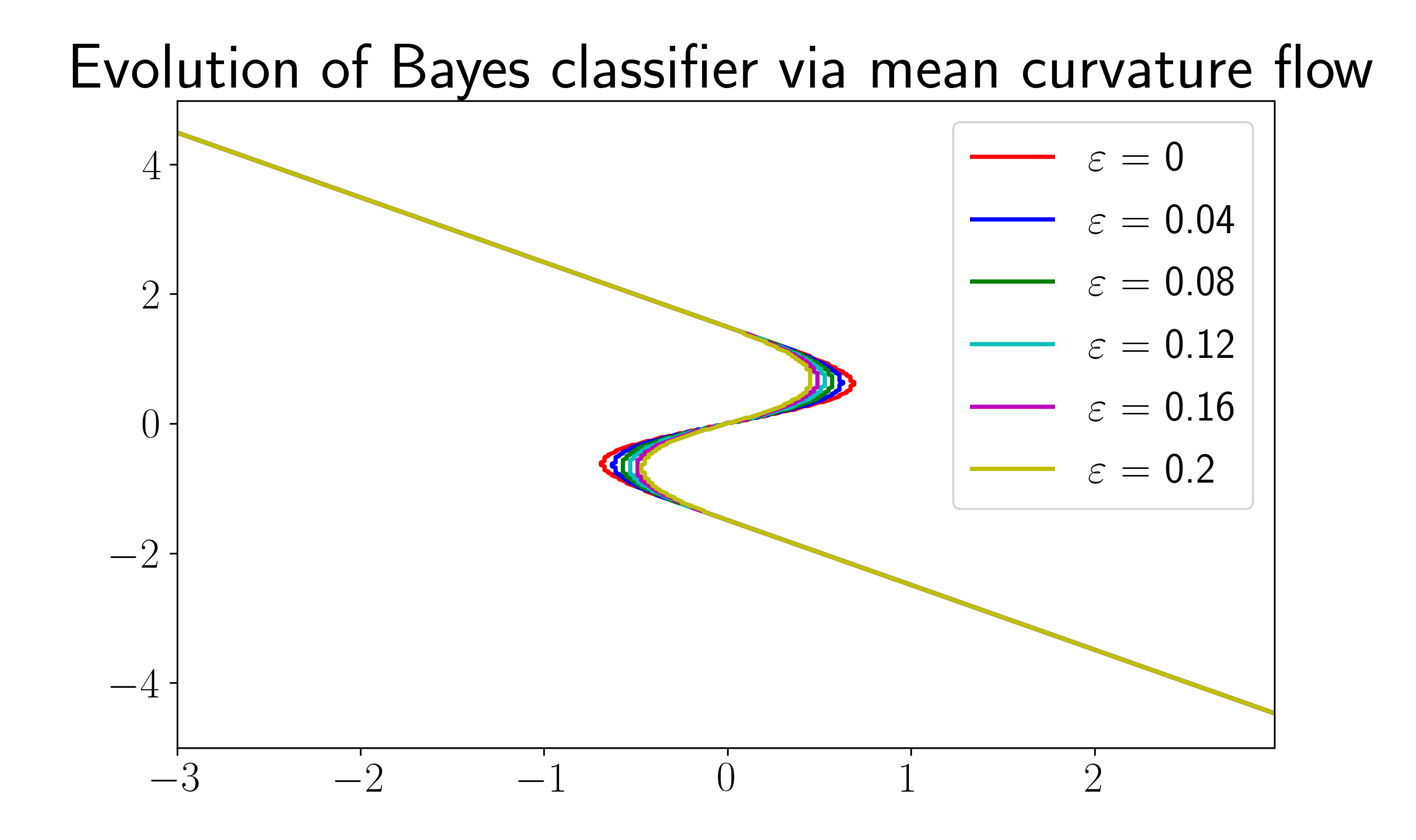}
\caption{The first set of curves represent the evolution of the decision boundary according to the geometric evolution \eqref{eqn:d-dim-evolution-approx}, which includes a weighted curvature flow and a drift term. The second set of curves is the geometric evolution following standard mean curvature flow. In this case the curves are largely the same, with only a very small damping of the curvature flow in the first case (which makes sense since $\nabla \rho$ is of modest size in this example).}
\label{fig:geometric-evolution}
\end{figure}

\section{Conclusion}\label{sec:conclusion}

This work provides a first analysis of the evolution equations associated with an ensemble of adversarial classification problems. In particular, we have shown that for the model considered here, the evolution equations in one dimension are completely able to characterize the global minimizer for small enough $\e$ (the power level of the adversary) without needing to conduct any optimization. In higher dimension the same evolution equations are linked with mean curvature flow and allude to implicit regularization.

This work suggests many promising future directions, both in terms of analysis and implementation. We list a few here, some of which are the topic of current investigation.

\begin{enumerate}
\item  \blue In this work, the connection between the evolution equations and global minimizers only holds for small $\e$, and we made no attempt to quantify the size of $\e_0$ in our theorem. This is partly unavoidable given the generality of the input distributions and the learned classifiers. We expect that the results in our paper should hold locally in $\e$, in the sense that if the adversarial problem admits a unique solution at $\tilde \e$ then the solution of the adversarial problem should be characterized by the evolution equation in a neighborhood of $\tilde \e$. We do not expect the theorem to generally hold for large ranges of $\e$, as topological changes can cause non-uniqueness of optimal solutions for certain (likely discrete) values of $\e$. It may be possible to use primal-dual methods to numerically detect when topological changes occur, or, in other words, in what ranges of $\e$ the evolution equations describe optimal solution families. Investigating such methods could provide a lot more information about how large $\e$ may be in our theorems.
\item In higher dimensions, we made the additional assumption that the evolution equations admitted smooth solutions. Justifying this assumption would likely require careful analysis of the partial differential equation. Similarly, it would be interesting to develop more efficient numerical methods for the actual evolution equation in higher-dimensional settings.\nc 
\item The smoothness of minimizers, and whether curvature is implicitly bounded, is a natural question. This is not obvious, as the objective functional of the adversarial problem does not impose a priori regularity. Similarly, the evolution of singularities (and whether they may disappear or appear) is completely unclear.
\item Various notions of distance have been used in studying adversarial examples. Notable examples include the $\ell_\infty$ distance. The effect of such a distance on the evolution equations that we describe in this work is an interesting question to study.
\item The problem of a data perturbing adversary for multiple labels, and the resulting evolution equations, is also a compelling, open problem.
\item Finally, here we have considered one specific example of adversarial classification model, but many others are possible. Likewise, we have restricted our attention to the classification problem with 0-1 loss, while one may also study other settings like regression under different loss functions.  Exploring other settings and studying their connection to other geometric flows is a promising direction of research that we hope to explore.  Our hope is to provide deeper insights into the properties of different robust learning methodologies. 
\end{enumerate}

\nc

\appendix
\section{Properties of the signed distance function}

We recall the definition of the signed distance function
\[
\tilde d_E(x) = \begin{cases} d(x,E) &\text{ if } x \not \in E \\ -d(x,E^c) &\text{ for } x \in E \end{cases}
\]

The following properties are classical and may be found in, e.g. \cite{ambrosio1998curvature}:

\begin{proposition}
\label{eq:PropsDistance}
Let $E$ be an open set with $C^2$ boundary. Then on some neighborhood $U$ of $\partial E$ we have the following:
\begin{itemize}
  \item $\tilde d \in C^2(U)$.
\item Each $y$ in $U$ has a unique closest point $P(y)$ in $\partial E$, and $P$ is a continuous function in $y$.
\item We have, for $y \notin \partial E$, that $\nabla \tilde d = \frac{P(y)-y}{\tilde d(Y)}$. For $y \in \partial E$ the outward unit normal is given by $\nu(y) = \nabla \tilde d(y)$.
\item For $y \in \partial E$, the matrix $D^2 \tilde d = \blue \frac{\partial \nu }{\partial x} \nc$ has $1$ eigenvalue that is equal to zero (with eigenvector in the normal direction $\nu$), and $d-1$ eigenvalues with eigenvectors spanning the tangent directions. \blue These eigenvalues are called the \emph{principal curvatures} of the surface, and are denoted $\kappa_i$. \nc
\end{itemize}
\end{proposition}

The principal curvatures of a surface may be viewed as inverses of the principal radii. The principal radii grow (or shrink depending on their sign) linearly in their distance from $\partial E$. By using these facts and applying a classical change of variables to the transformation $T(x) = x + \e \nu(x)$, we obtain the following formula:
\begin{corollary}\label{cor:change-of-variables}
 If $\partial E$ is a $C^2$ surface then for $\e$ sufficiently small
  \[
    \int_{\tilde d_A(y) = \e} g(y) \,d\mathcal{H}^{d-1}(y) = \int_{\tilde d_A(x) = 0} g(x + \e \nu(x)) \prod_{i=1}^{d-1}|1 + \e \kappa_i(x)| \,d\mathcal{H}^{d-1}(x).
  \]
\end{corollary}

\blue 
The following lemma, sometimes called the ``layer cake representation'', is a classical lemma from measure theory (see e.g. Chapter 1 in \cite{lieb2001analysis}), and may be directly proved by using the Fubini-Tonelli theorem. 
\begin{lemma}\label{lem:layer-cake}
Given a non-negative function $f$ on a measure space $(X,\mu)$ the following identity holds:
\[
\int_X f(x) \,d\mu(x) = \int_0^\infty \mu(\{ x : f(x) > t\}) \,dt.
\]
\end{lemma}
\nc

%
%

\bibliographystyle{alpha}
\bibliography{AML}

\end{document}